\documentclass[a4paper]{article}

\usepackage{arxiv}

\usepackage[utf8]{inputenc} 
\usepackage[T1]{fontenc}    
\usepackage{hyperref}       
\hypersetup{
  colorlinks  = true,       
  urlcolor    = darkgray,   
  linkcolor   = darkgray,   
  citecolor   = darkgray    
}

\usepackage{url}            
\usepackage{booktabs}       
\usepackage{amsfonts}       
\usepackage{nicefrac}       
\usepackage{microtype}      
\usepackage{multirow}       

\usepackage{amssymb,amsthm,amsmath}
\usepackage{textcomp}
\usepackage[shortlabels]{enumitem}
\usepackage{tikz}
\usetikzlibrary{positioning}

\providecommand{\st}{}
\renewcommand{\st}{\mathrel{\mid}}
\newcommand{\ldbl}{\{\!\!\{}
\newcommand{\rdbl}{\}\!\!\}}

\newcommand{\architecture}{\mathcal{M}}
\newcommand{\architectureWL}{\mathcal{M}_{\textsl{WL}}}
\newcommand{\architectureano}{\mathcal{M}_{\textsl{anon}}}

\newtheorem{theorem}{Theorem}[section]
\newtheorem{lemma}[theorem]{Lemma}
\newtheorem{proposition}[theorem]{Proposition}
\newtheorem{corollary}[theorem]{Corollary}

\theoremstyle{definition}
\newtheorem{definition}[theorem]{Definition}
\newtheorem{example}[theorem]{Example}
\newtheorem{remark}[theorem]{Remark}

\title{Let's Agree to Degree: Comparing Graph Convolutional Networks
in the Message-Passing Framework}

\author{
  Floris Geerts\\
  University of Antwerp\\
  \texttt{floris.geerts@uantwerpen.be}
  \And
  Filip Mazowiecki\\
  Max Planck Institute for Software Sciences\\
  \texttt{filipm@mpi-sws.org}
  \And
  Guillermo A. P\'erez\\
  University of Antwerp\\
  \texttt{guillermoalberto.perez@uantwerpen.be}
}
\date{}
\begin{document}

\maketitle

\begin{abstract}
In this paper we cast neural networks defined on graphs as message-passing neural
networks (MPNNs) in order to study the distinguishing power of different classes
of such models. We are interested in whether certain architectures are able to
tell vertices apart based on the feature labels given as input with the graph. We
consider two variants of MPNNS: anonymous MPNNs whose message functions depend
only on the labels of vertices involved; and degree-aware MPNNs in which message
functions can additionally use information regarding the degree of vertices. The
former class covers a popular formalisms for computing functions on graphs: graph neural
networks (GNN). The latter covers the so-called graph convolutional networks
(GCNs), a recently introduced variant of GNNs by Kipf and Welling. We obtain
lower and upper bounds on the distinguishing power of MPNNs in terms of the
distinguishing power of the Weisfeiler-Lehman (WL) algorithm. Our results imply
that (i)~the distinguishing power of GCNs is bounded by the WL algorithm, but
that they are one step ahead; (ii)~the WL algorithm cannot be simulated by
``plain vanilla'' GCNs but the addition of a trade-off parameter between features
of the vertex and those of its neighbours (as proposed by Kipf and Welling
themselves) resolves this problem.
\end{abstract}

\section{Introduction}\label{sec:intro}
A standard approach to learning tasks on graph-structured data, such as vertex
classification, edge prediction, and graph classification, consists of the
construction of a representation of vertices and graphs that captures
their structural information. Graph Neural Networks (GNNs) are currently
considered as the state-of-the art approach for learning such representations.
Many variants of GNNs exist but they all follow a similar strategy. More
specifically, each vertex is initially associated with a feature vector. This
is followed by a recursive neighbourhood-aggregation scheme where each vertex
aggregates feature vectors of its neighbours, possibly combines this with its
own current feature vector, to finally obtain its new feature vector. After a
number of iterations, each vertex is then represented by the resulting feature
vector.

The adequacy of GNNs for graph learning tasks is directly related to their
so-called distinguishing power. Here, distinguishing power refers to
the ability of GNNs to distinguish vertices and graphs in terms of the computed
representation. That is, when two vertices are represented by the same feature
vector, they are considered the same with regards to any subsequent
feature-based task. 

Only recently a formal study of the distinguishing power of some GNN variants
has been initiated. In two independent studies~\cite{xhlj19,grohewl} the
distinguishing power of GNNs is linked to the distinguishing power of the
classical Weisfeiler-Lehman (WL) algorithm. The WL algorithm starts
from an initial vertex colouring of the graph. Then, similarly as GNNs, the WL
algorithm recursively aggregates the colouring of neighbouring vertices. In
each recursive step, a vertex colouring is obtained that refines the previous
one. The WL algorithm stops when no further refinement is obtained. The
distinguishing power of the WL algorithm itself is well understood, see
e.g.,~\cite{CaiFI92,KieferSS15,ArvindKRV17}.

In~\cite{xhlj19,grohewl} it is shown that for any input graph if vertices can be
distinguished by a GNN then they can be distinguished by the WL algorithm.
Conversely, Graph Isomorphism Networks (GINs) were proposed in~\cite{xhlj19}
that can match the distinguishing power of the WL algorithm, on any graph. The
construction of GINs relies on multi-layer perceptrons and their ability to
approximate arbitrary functions. In contrast, \cite{grohewl} show that the
distinguishing power of the WL algorithm can also be matched by using GNNs,
provided that the input graph is fixed. Both these works consider undirected
vertex-labelled graphs. We remark that the work by~\cite{grohewl} has recently
been extended to directed graphs, possibly with vertex- and
edge-labels~\cite{Jaume2019}. We refer to~\cite{Sato2020ASO} for an in-depth
survey on the expressive power of graph neural networks.

In this paper we start from the observation that many popular GNNs fall outside
of the class of GNNs considered in previous work~\cite{xhlj19,grohewl}.
Prominent examples of such GNNs are the so-called Graph Convolutional Networks
(GCNs)~\cite{kipf-loose}. Although GCNs adhere to the same strategy as GNNs
(i.e., recursive neighbourhood aggregation of features), they additionally take
into account \textit{vertex-degree information}. In this paper, we continue the
study of the distinguishing power of large classes of GNNs that may use degree
information.

To do so, we leverage connections between GNNs, GCNs and so-called Message
Passing Neural Network (MPNNs) introduced by~\cite{GilmerSRVD17}. Such neural
networks perform a number of rounds of computation, and in each such round,
vertex labels are propagated along the edges of the graph and aggregated at the
vertices. MPNNs are known to encompass many GNN and GCN
formalisms~\cite{GilmerSRVD17}. We refer to~\cite{Zhou2018,Zonghan2019} for
extensive surveys on GNNs, GCNs and MPNNs.

The general MPNN framework allows us to explore the impact of degree information
on the distinguishing power of MPNNs in general and large classes of GNNs and
GCNs in particular. More precisely, in this paper we consider two general
classes of MPNNs: \textit{anonymous} MPNNs that do not use degree information,
and \textit{degree-aware} MPNNs that do use degree information. The former
class of MPNNs covers the GNNs studied in~\cite{xhlj19,grohewl}, the latter
class covers the GCNs~\cite{kipf-loose}, among others.

\paragraph{Contributions.} 
For general MPNNs, our main results are the following (see Propositions~\ref{pro:eqstrongWL} and~\ref{prop:onestep}):
\begin{enumerate}[(i)]
    \item The distinguishing power of anonymous MPNNs is bounded by the WL algorithm. This result can be seen as a slight generalisation of the results in~\cite{xhlj19,grohewl}.
	\item The distinguishing power of degree-aware MPNNs is bounded by the WL algorithm, \textit{but they may be one step ahead}.  Intuitively, degree-aware MPNNs may be one step ahead of the WL algorithm because the degree information, which is part of degree-aware MPNNs from the start, is only derived by the WL algorithm after one step.
	\item The WL algorithm can be regarded as an anonymous MPNN (and thus also as a degree-aware MPNN). As a consequence, the distinguishing power of the classes of anonymous and degree-aware MPNNs matches that of the WL algorithm.
\end{enumerate}

For anonymous MPNNs related to GNNs~\cite{xhlj19,grohewl} and degree-aware MPNNs
related to GCNs~\cite{kipf-loose}, our main results are the following (see
Theorems~\ref{thm:grohe_lower} and~\ref{thm:equalstrong}, and
Propositions~\ref{prop:notasstrong} and~\ref{prop:indeed-wl-power}):
\begin{enumerate}[resume*]
\item On a fixed input graph, the WL algorithm can be simulated, step-by-step, by GNNs that use ReLU or sign as activation function. This result refines the result in~\cite{grohewl} in that their simulation using the ReLU function requires two GNN ``layers'' for each step of the WL algorithm. We only require one layer in each step. In addition, our simulation is achieved by means of a very simple form of GNNs (see Equation~\eqref{eq:GNNWL} at the end of Section~\ref{sec:anonymous}), which may be of independent interest.
\item  The distinguishing power of GCNs is bounded by the WL algorithm, \textit{but they may be  one step ahead}. This is due to GCNs being degree-aware MPNNs (for which result~(ii) applies). This advantage of GCNs over more classical GNNs may explain the success of GCNs in various graph learning tasks.
\item In contrast, we show that the WL algorithm \textit{cannot} be simulated by popular GCNs such as those from~\cite{kipf-loose}. This observation is somewhat contradictory to the general belief that GCNs can be seen as a ``continuous generalisation'' of the WL algorithm.
\item However, by introducing a learnable trade-off parameter between features of the vertex and those of its neighbours, the simulation of the WL algorithm can be achieved by GCNs. This minor relaxation of GCNs (see Equation~\eqref{gnn:kipfp} at the end of Section~\ref{sec:dMPNNs}) was already suggested in~\cite{kipf-loose} based on empirical results.
Our simulation result thus provides a theoretical justification of this parameter. 
\end{enumerate}

\paragraph{Structure of the paper.} 
After introducing some notations and concepts in Section~\ref{sec:prelims}, we define MPNNs, anonymous and degree-aware MPNNs in Section~\ref{sec:MPNNs}. In Section~\ref{sec:compare} we formally define how to compare classes of MPNNs with regard to their distinguishing power. We characterise the distinguishing power of anonymous MPNNs in Section~\ref{sec:anonymous} and that of degree-aware MPNNs in Section~\ref{sec:dMPNNs}. We conclude the paper in Section~\ref{sec:conclude}.

\section{Preliminaries}\label{sec:prelims}
Let $\mathbb{A}$ denote the set of all algebraic numbers; $\mathbb{Q}$, the set
of all rational numbers; $\mathbb{Z}$, the set of all integer numbers;
$\mathbb{N}$, the set of all natural numbers including zero, i.e., $\mathbb{N} =
\{0,1,2,\dots\}$. We write $\mathbb{S}^+$ to denote the subset of numbers from
$\mathbb{S}$ which are strictly positive, e.g., $\mathbb{N}^+ = \mathbb{N}
\setminus \{0\}$. We use $\{\!\}$ and $\ldbl\! \rdbl$ to indicate sets and
multisets, respectively.

\paragraph{Computing with algebraic numbers.} 
Throughout the paper we will perform basic computations, such as addition and
multiplication, on numbers. It is well-known that these operations are
computable on numbers in $\mathbb{N}$, $\mathbb{Z}$ and $\mathbb{Q}$. However,
in order to capture numbers used by popular graph neural network architectures,
such as roots of integers~\cite{kipf-loose}, we will work with \emph{algebraic
numbers}. An algebraic number is usually represented by a minimal polynomial
such that the number is a root of the polynomial and a pair of rational numbers
to identify that root. Conveniently, it is known that the operations we will
need are indeed computable for algebraic numbers encoded using such a
representation (see, e.g.,~\cite{OuaknineW14}).

\paragraph{Labelled graphs.}
Let $G=(V,E)$ be an undirected graph consisting of $n \in \mathbb{N}$ vertices.
Without loss of generality we assume that $V=\{1,2,\dots,n\}$. Given a vertex
$v\in V$, we denote by $N_G(v)$ its set of neighbours, i.e., $N_G(v):=\{u\st
\{u,v\}\in E\}$. Furthermore, the degree of a vertex $v$, denoted by $d_{v}$, is
the number of vertices in $N_G(v)$. With a labelled graph $(G,\pmb{\nu})$ we
mean a graph $G=(V,E)$ whose vertices are labelled using a function
$\pmb{\nu}:V\to \Sigma$ for some set $\Sigma$ of labels. We denote by
$\pmb{\nu}_v$ the label of $v$ for each $v\in V$.

Henceforth we fix a labelled graph $(G,\pmb{\nu})$ with $G=(V,E)$ and denote by
$\mathbf{A}$ the adjacency matrix (of $G$). That is, $\mathbf{A}$ is a matrix of
dimension $n \times n$ such that the entry $\mathbf{A}_{vw}=1$ if $\{v,w\}\in E$
and $\mathbf{A}_{vw}=0$ otherwise. We denote by $\mathbf{D}$ the diagonal matrix
such that $\mathbf{D}_{vv}=d_v$ for each $v\in V$. Throughout the paper we will
assume that $G$ does not have isolated vertices, which is equivalent to assuming
that $\mathbf{D}$ does not have any $0$ entries on the diagonal. This assumption
will help us to avoid unnecessary technical details in the theoretical analysis.
But it is easy to generalise our results by treating isolated nodes separately.
We will also assume that there are no self-loops, so the diagonal of
$\mathbf{A}$ is filled with $0$s. For an arbitrary matrix $\mathbf{B}$ we denote
by $\mathbf{B}_{i}$ the $i$-th row of $\mathbf{B}$. Furthermore, if $\mathbf{B}$
is a matrix of dimension $n\times m$, we also represent the rows of $\mathbf{B}$
by $\mathbf{B}_{v}$, for $v\in V$.

We will identify $\Sigma$ with elements (row vectors) in $\mathbb{A}^s$ for some
$s\in\mathbb{N}^+$. In this way, a labelling $\pmb{\ell}:V\to\Sigma$ can be
regarded as a matrix in $\mathbb{A}^{n\times s}$ and $\pmb{\ell}_v$ corresponds
to the $v$-th row in that matrix. Conversely, a matrix $\mathbf{L} \in
\mathbb{A}^{n \times s}$ can be regarded as the vertex labelling that labels $v$
with the row vector $\mathbf{L}_{v}$. We use these two interpretations of
labellings interchangeably.

It will be important later on to be able to compare two labellings of $G$. Given
a matrix $\mathbf{L} \in \mathbb{A}^{n\times s}$ and a matrix $\mathbf{L}' \in
\mathbb{A}^{n\times s'}$ we say that the vertex labelling $\mathbf{L}'$ is
coarser than the vertex labelling $\mathbf{L}$, denoted by
$\mathbf{L}\sqsubseteq \mathbf{L}'$, if for all $v,w\in V$, $\mathbf{L}_{v}=\mathbf{L}_{w} \Rightarrow \mathbf{L}'_{v}=\mathbf{L}'_{w}$. The
vertex labellings $\mathbf{L}$ and $\mathbf{L}'$ are equivalent, denoted by
$\mathbf{L}\equiv\mathbf{L}'$, if $\mathbf{L}\sqsubseteq \mathbf{L}'$ and
$\mathbf{L}'\sqsubseteq \mathbf{L}$ hold. In other words,
$\mathbf{L}\equiv\mathbf{L}'$ if and only if for all $v,w\in V$, 
$\mathbf{L}_{v}=\mathbf{L}_{w} \Leftrightarrow \mathbf{L}'_{v}=\mathbf{L}'_{w} $.

\paragraph*{Weisfeiler-Lehman labelling.}
Of particular importance is the labelling obtained by colour refinement, also
known as the Weisfeiler-Lehman algorithm (or WL algorithm, for short). The WL
algorithm constructs a labelling, in an incremental fashion, based on
neighbourhood information and the initial vertex labelling. More specifically,
given $(G,\pmb{\nu})$, the WL algorithm initially sets
$\pmb{\ell}^{(0)}:=\pmb{\nu}$. Then, the WL algorithm computes a labelling
$\pmb{\ell}^{(t)}$, for $t> 0$, as follows:
\begin{equation*}
    \pmb{\ell}^{(t)}_v:=\textsc{Hash}\Bigl(\bigl(\pmb{\ell}^{(t-1)}_v,\ldbl
    \pmb{\ell}_u^{(t-1)} \st u \in N_G(v) \rdbl\bigr)\Bigr),
\end{equation*}
where $\textsc{Hash}$ bijectively maps the above pair, consisting of (i)~the
previous label $\pmb{\ell}^{(t-1)}_v$ of $v$; and (ii)~the multiset $\ldbl
\pmb{\ell}_u^{(t-1)} \st u \in N_G(v) \rdbl$ of labels of the neighbours of $v$,
to a label in $\Sigma$ which has not been used in previous iterations. When the
number of distinct labels in $\pmb{\ell}^{(t)}$ and $\pmb{\ell}^{(t-1)}$ is the
same, the WL algorithm terminates. Termination is guaranteed in at most $n$
steps~\cite{Immerman1990}.
 
\section{Message Passing Neural Networks}\label{sec:MPNNs}
We start by describing message passing neural networks (MPNNs) for  deep
learning on graphs, introduced by \cite{GilmerSRVD17}. Roughly speaking, in
MPNNs, vertex labels are propagated through a graph according to its
connectivity structure. MPNNs are known to model a variety of graph neural
network architectures commonly used in practice. We define MPNNs
in Section~\ref{subsec:def}, provide some examples in Section~\ref{subsec:examples},
and comment on the choice of formalisation of MPNNs in Section~\ref{subsec:comments}.

\subsection{Definition}\label{subsec:def}
Given a labelled graph $( G,\pmb{\nu})$ and a computable function $f : V \to \mathbb{A}$
an MPNN computes a vertex
labelling $\pmb{\ell}:V\to \mathbb{A}^{s}$, for some $s\in\mathbb{N}^+$.
The vertex labelling computed by an MPNN is computed in a finite number of rounds $T$.
After round $0 \le t \le T$ the labelling is denoted by
$\pmb{\ell}^{(t)}$. We next detail how $\pmb{\ell}^{(t)}$ is computed.
\begin{description}\setlength{\itemsep}{-0.4ex}
\item [Initialisation.]  We let $\pmb{\ell}^{(0)}:=\pmb{\nu}$.
\end{description}
Then, for every round $t=1,2,\ldots,T$, we define
$\pmb{\ell}^{(t)}:V\to\mathbb{A}^{s_t}$, as follows\footnote{Note that
we allow for labels to have different dimensions $s_t \in \mathbb{N}^+$
per round $t$.}:
\begin{description}\setlength{\itemsep}{-0.4ex}
\item [Message Passing.] 
Each vertex $v\in V$ receives messages from its neighbours which are
subsequently aggregated. Formally, the function $\textsc{Msg}^{(t)}$ receives as
input $f$ applied to two vertices $v$ and $u$, and the corresponding labels of
these vertices from the previous iteration $\pmb{\ell}^{(t-1)}_v$ and
$\pmb{\ell}^{(t-1)}_u$, and outputs a label in $\mathbb{A}^{s'_t}$. Then, for
every vertex $v$, we aggregate by summing all such labels for every neighbour
$u$.
$$
\mathbf{m}^{(t)}_{v}:=\sum_{u\in
N_G(v)}\textsc{Msg}^{(t)}\left(\pmb{\ell}^{(t-1)}_v,\pmb{\ell}^{(t-1)}_u,f(v),f(u)\right)\in\mathbb{A}^{s'_t}.
$$
\item [Updating.] Each vertex $v\in V$ further updates $\mathbf{m}{}^{(t)}_{v}$
  possibly based on its current label $\pmb{\ell}^{(t-1)}_v$:
$$
\pmb{\ell}^{(t)}_v:=\textsc{Upd}^{(t)}\left(\pmb{\ell}^{(t-1)}_v,\mathbf{m}^{(t)}_{v}\right)\in\mathbb{A}^{s_t}.
$$
\end{description}
Here, the message functions $\textsc{Msg}^{(t)}$ and update functions
$\textsc{Upd}^{(t)}$ are general (computable) functions. After round $T$, we
define the final labelling $\pmb{\ell}:V\to\mathbb{A}^{s}$ as
$\pmb{\ell}_v:=\pmb{\ell}^{(T)}_v$ for every $v\in V$. If further aggregation
over the entire graph is needed, e.g., for graph classification, an additional
readout function $\textsc{ReadOut}(\ldbl\pmb{\ell}_v\mid v\in V\rdbl)$ can be
applied. We omit the readout function here since most of the computation happens
during the rounds of an MPNN.

The role of the function $f$ in this paper is to distinguish between two classes
of MPNNs\footnote{In general, one could consider any function $f$.}: Those whose
message functions only depend on the labels of the vertices involved, in which
case we set $f$ to the zero function $f(v)=0$, for all $v\in V$, and those whose
message functions depend on the labels and on the degrees of the vertices
involved, in which case we set $f$ to the degree function $f(v)=d_v$, for all
$v\in V$. We will refer to the former class as \textit{anonymous} MPNNs and to
the latter as \textit{degree-aware} MPNNs. These classes are denoted by
$\architectureano$ and $\architecture_{\textsl{deg}}$, respectively. We remark
that, by definition, $\architectureano\subseteq \architecture_{\textsl{deg}}$.

\subsection{Examples}\label{subsec:examples}
We next illustrate anonymous and degree-aware MPNNs by a number of examples.
First, we provide two examples of anonymous MPNNs that will play an important role in Section~\ref{sec:anonymous}.

\begin{example}[GNN architectures]\label{ex:GNN}
We first consider the graph neural network architectures~\cite{hyl17,grohewl}
defined by:
\begin{equation}
\mathbf{L}^{(t)}:=\sigma\left(\mathbf{L}^{(t-1)}\mathbf{W}_1^{(t)}+\mathbf{A}\mathbf{L}^{(t-1)}\mathbf{W}_2^{(t)}+\mathbf{B}^{(t)}\right), \label{gnn:grohe}
\end{equation}
where $\mathbf{L}^{(t)}$ is the matrix in $\mathbb{A}^{n\times s_t}$ consisting
of the $n$ rows $\pmb{\ell}^{(t)}_v\in\mathbb{A}^{s_t}$, for $v\in V$,
$\mathbf{A}\in\mathbb{A}^{n\times n}$ is the adjacency matrix of $G$,
$\mathbf{W}_1^{(t)}$ and $\mathbf{W}_2^{(t)}$ are (learnable) weight matrices in
$\mathbb{A}^{s_{t-1}\times s_t}$, $\mathbf{B}^{(t)}$ is a bias matrix in
$\mathbb{A}^{n\times s_t}$ consisting of $n$ copies of the same row
$\mathbf{b}^{(t)}\in \mathbb{A}^{s_t}$, and $\sigma$ is a non-linear activation
function. We can regard this architecture as an MPNN. Indeed,~(\ref{gnn:grohe})
can be equivalently phrased as the architecture which computes, in round $t$,
for each vertex $v\in V$ the label defined by:
\[
\pmb{\ell}^{(t)}_v:=\sigma\left(\pmb{\ell}^{(t-1)}_v\mathbf{W}_1^{(t)}+ \sum_{u\in N_G(v)}\pmb{\ell}^{(t-1)}_u\mathbf{W}_2^{(t)}+\mathbf{b}^{(t)} \right),
\]
where we identified the labellings with their images, i.e., a row vector in
$\mathbb{A}^{s_{t-1}}$ or $\mathbb{A}^{s_t}$. To phrase this as an MPNN, it
suffices to define for each $\mathbf{x}$ and $\mathbf{y}$ in
$\mathbb{A}^{s_{t-1}}$, each $v\in V$ and $u\in N_G(v)$, and each $t\geq 1$:
\begin{equation*}
	\textsc{Msg}^{(t)}\bigl(\mathbf{x},\mathbf{y},-,-):=\mathbf{y}\mathbf{W}_2^{(t)}
\text{ and } 
\textsc{Upd}^{(t)}(\mathbf{x},\mathbf{y}):=\sigma\left(\mathbf{x}\mathbf{W}_1^{(t)}+\mathbf{y} + \mathbf{b}^{(t)}\right).
\end{equation*} 

We write $-$ instead of $0$ to emphasise that the message functions use the zero
function $f(v)=0$, for all $v\in V$, and hence do not depend on $f(v)$ and
$f(u)$. In other words, the MPNN constructed is an anonymous MPNN. Without loss
of generality we will assume that aMPNNs do not use $f(v)$ and $f(u)$ in the
messages. If they do then one can replace them with $0$. This way it is easy to
see that classes of MPNNs that use different functions $f$ in the messages
contain the class of anonymous MPNNs. \hfill$\blacksquare$
\end{example}
Another example of an anonymous MPNN originates from the Weisfeiler-Lehman
algorithm described in the preliminaries.
\begin{example}[Weisfeiler-Lehman]\label{ex:WL}
We recall that WL computes, in round $t \geq 1$, for each vertex $v\in V$ the label:
$$
\pmb{\ell}^{(t)}_v:=\textsc{Hash}\left(\pmb{\ell}^{(t-1)}_v,\ldbl \pmb{\ell}_u^{(t-1)} \st u \in N_G(v) \rdbl\right).
$$
Let us assume that the set $\Sigma$ of labels is $\mathbb{A}^s$ for some fixed
$s\in\mathbb{N}^+$. We cast the WL algorithm as an anonymous MPNN by using an
injection $h : \mathbb{A}^s \to \mathbb{Q}$. What follows is in fact an
adaptation of Lemma 5 from~\cite{xhlj19} itself based on~\cite[Theorem
2]{ZaheerKRPSS17}. We crucially rely on the fact that the set $\mathbb{A}$ of
algebraic numbers is countable (see e.g., Theorem 2.2. in~\cite{Frazer}). As a
consequence, also $\mathbb{A}^s$ is countable.

Let $\tau : \mathbb{A}^s \to \mathbb{N}^+$ be a computable injective function
witnessing the countability of $\mathbb{A}^s$. For instance, since elements of
$\mathbb{A}$ are encoded as a polynomial $a_0 + a_1 x^1 + a_2 x^2 + \dots +
a_kx^k \in \mathbb{Z}[x]$ and a pair $\nicefrac{n_1}{d_1},\nicefrac{n_2}{d_2}$
of rationals, $\tau$ can be taken to be the composition of the injection $\alpha
: \mathbb{A} \to \mathbb{N}^+$, applied point-wise, and the Cantor tuple
function, where
\[
    \alpha(a_0 + a_1 x^1 + a_2 x^2 + \dots + a_kx^k,\nicefrac{n_1}{d_1},\nicefrac{n_2}{d_2}) \mapsto
    p(1,n_1) p(2,n_2) p(3,d_1) p(4,d_2) \prod_{i=0}^k p(i+5,a_i)
\]
with $\pi_i$ being the $i$-th prime number in
\[
    p(i,z) = \begin{cases}
        \pi_{2i}^z & \text{if } z \geq 0\\
        \pi_{2i+1}^{-z} & \text{if } z < 0.
    \end{cases}
\]
We next define $h:\mathbb{A}^s\to\mathbb{Q}^+$ as the mapping $\mathbf{x}
\mapsto (n+1)^{-\tau(\mathbf{x})}$. Note that $h$ is injective and
$h(\mathbf{x})$ can be seen as a number whose $(n+1)$-ary representation has a
single nonzero digit. We next observe that the multiplicity of every element in
$S := \ldbl \pmb{\ell}_u^{(t-1)} \st u \in N_G(v) \rdbl$ is bounded by the
number of all vertices $n$ --- and this for all $t \geq 1$. It follows that the
function $\phi$ mapping any such $S$ to $\sum_{\mathbf{x} \in S} h(\mathbf{x})$
is an injection from $\mathbb{A}^s$ to $\mathbb{Q}$. Therefore, the summands can
be recovered by looking at the $(n+1)$-ary representation of the sum, and thus
the inverse of $\phi$ is computable on its image. To conclude, we define the
message function
\[
\textsc{Msg}^{(t)}(\mathbf{x},\mathbf{y},-,-) :=
h(\mathbf{y}) \in \mathbb{A}.
\]
The update function is defined by
\[
\textsc{Upd}^{(t)}(\mathbf{x},y) :=
\textsc{Hash}(\mathbf{x},\phi^{-1}(y)),
\]
where $y \in \mathbb{A}$ since it corresponds to a sum of messages, themselves
algebraic numbers. As before, we write $-$ instead of $0$ to emphasise that the
message functions use the zero function $f(v)=0$, for all $v\in V$, and hence do
not depend on $f(v)$ and $f(u)$.\hfill$\blacksquare$
\end{example}

We conclude with an example of a degree-aware MPNN. We  study degree-aware MPNNs in Section~\ref{sec:dMPNNs}.
\begin{example}[GCNs by Kipf and Welling]\label{ex:KipfasMPNN}
We consider the GCN architecture by~\cite{kipf-loose}, which in round $t \geq 1$ computes
\begin{equation*}
\mathbf{L}^{(t)} := \sigma\left(
(\mathbf{D}+\mathbf{I})^{-1/2}(\mathbf{A}+\mathbf{I})(\mathbf{D}+\mathbf{I})^{-1/2}\mathbf{L}^{(t-1)}\mathbf{W}^{(t)}
  \right),
\end{equation*}
where we use the same notation as in Example~\ref{ex:GNN} but now with a single
(learnable) weight matrix $\mathbf{W}^{(t)}$ in $\mathbb{A}^{s_{t-1}\times
s_t}$. This means that, in round $t$, for each vertex $v\in V$ it computes the
label:
\begin{equation}
\pmb{\ell}^{(t)}_v:=\sigma\left(\left(\frac{1}{1+d_v}\right)\pmb{\ell}_v^{(t-1)}\mathbf{W}^{(t)} + \sum_{u\in N_G(v)} \left(\frac{1}{\sqrt{1+d_v}}\right)\left(\frac{1}{\sqrt{1+d_u}}\right)\pmb{\ell}^{(t-1)}_u\mathbf{W}^{(t)}\right). \label{GNN:Kipf}
\end{equation}
We can regard this architecture again as an MPNN. Indeed, it suffices to define
for each $\mathbf{x}$ and $\mathbf{y}$ in $\mathbb{A}^{s_{t-1}}$, each $v\in V$
and $u\in N_G(v)$, and each $t\geq 1$:
\begin{align*}
\textsc{Msg}^{(t)}\left(\mathbf{x},\mathbf{y},d_v,d_u\right)&:=
\frac{1}{d_v}\left(\frac{1}{1+d_v}\right)\mathbf{x}\mathbf{W}^{(t)}+
\left(\frac{1}{\sqrt{1+d_v}}\right)\left(\frac{1}{\sqrt{1+d_u}}\right)\mathbf{y}\mathbf{W}^{(t)}
\intertext{and} \textsc{Upd}^{(t)}(\mathbf{x},\mathbf{y})&:=\sigma(\mathbf{y}).
\end{align*}
We remark that the initial factor $1/d_v$ in the message functions is introduced
for renormalisation purposes. We indeed observe that the message functions
depend only on $\pmb{\ell}^{(t-1)}_v$, $\pmb{\ell}^{(t-1)}_u$, and the degrees
$d_v$ and $d_u$ of the vertices $v$ and $u$, respectively.\hfill$\blacksquare$
\end{example}

\subsection{On the choice of formalism}\label{subsec:comments}
The expert reader may have noticed that we use a different formalisation of
MPNNs than the one given in the original paper~\cite{GilmerSRVD17}. The first
difference is that our MPNNs are parameterised by an input computable
function~$f$ applied to $v$ and $u \in N_G(v)$. We add this function to avoid a
certain ambiguity in the formalisation in ~\cite{GilmerSRVD17} on what precisely
the message functions can depend on. More specifically, only a dependence on
$\pmb{\ell}_v^{(t-1)}$ and $\pmb{\ell}_u^{(t-1)}$ is specified
in~\cite{GilmerSRVD17}. In contrast, the examples given in~\cite{GilmerSRVD17}
use more information, such as the degree of vertices. The use of the function
$f$ in the definition of MPNNs makes explicit the information that message
functions can use. It is readily verified that every MPNN of~\cite{GilmerSRVD17}
corresponds to an MPNN in our formalism.

The second difference is that the MPNNs in ~\cite{GilmerSRVD17} work on graphs
that carry both vertex and edge labels. We ignore edge labellings in this paper
but most of our results carry over to that more general setting. Indeed, it
suffices to use the extension of the Weisfeiler-Lehman algorithm for
edge-labelled graphs as is done for graph neural networks in~\cite{Jaume2019}.

We also want to compare our formalisation to the MPNNs from~\cite{Loukas2019}.
In that paper, the message functions can depend on identifiers of the vertices
involved. Such position-aware MPNNs correspond to MPNNs in our setting in which
$f$ assigns to each vertex a unique identifier. We remark that~\cite{Loukas2019}
shows Turing universality of position-aware MPNNs using close connections with
the LOCAL model for distributed graph computations of~\cite{Angluin}. As such,
MPNNs from~\cite{Loukas2019} can simulate our MPNNs as one could add a few
initialisation rounds to compute $f(v)$ and $f(u)$. We also remark that in the
MPNNs from~\cite{Loukas2019} every vertex can also send itself a message. We
provide this functionality by parameterising the update functions with the
current label of the vertex itself, just as in~\cite{GilmerSRVD17}.

\section{Comparing the distinguishing power of classes of MPNNs}\label{sec:compare}
The distinguishing power of MPNNs relates to their ability to distinguish
vertices based on the labellings that they compute. We are interested in
comparing the distinguishing power of classes of MPNNs. In this section we
formally define what we mean by such a comparison.

For a given labelled graph $( G,\pmb{\nu})$ and MPNN $M$, we denote by
$\pmb{\ell}_M^{(t)}$ the vertex labelling computed by $M$ after $t$ rounds. We
fix the input graph in what follows, so we do not need to include the dependency
on the graph in the notation of labellings.

\begin{definition}\label{def:mpnnweak}\normalfont
Consider two MPNNs $M_1$ and $M_2$ with the same number of rounds $T$. Let
$\pmb{\ell}_{M_1}^{(t)}$ and $\pmb{\ell}_{M_2}^{(t)}$ be their corresponding
labellings on an input graph $( G,\pmb{\nu})$ obtained after $t$ rounds of
computation for every $0 \le t \le T$. Then $M_1$ is said to be \textit{weaker}
than $M_2$, denoted by $M_1\preceq M_2$, if $M_1$ cannot distinguish more
vertices than $M_2$ in every round of computation. More formally, $M_1\preceq
M_2$ if $\pmb{\ell}_{M_2}^{(t)}\sqsubseteq \pmb{\ell}_{M_1}^{(t)}$ for every
$t\geq 0$. In this case we also say that $M_2$ is \textit{stronger} than
$M_1$.\hfill$\blacksquare$
\end{definition}
We lift this notion to classes $\architecture_1$ and $\architecture_2$ of MPNNs in a
standard way.

\begin{definition}\label{def:classesweak}\normalfont
Consider two classes $\architecture_1$ and $\architecture_2$ of MPNNs. Then,
$\architecture_1$ is said to be \textit{weaker} than $\architecture_2$, denoted
by $\architecture_1\preceq \architecture_2$, if for all $M_1\in \architecture_1$
there exists an $M_2\in\architecture_2$ which is stronger than $M_1$.
\hfill$\blacksquare$
\end{definition}

Finally, we say that $\architecture_1$ and $\architecture_2$ are \textit{equally
strong}, denoted by $\architecture_1\equiv \architecture_2$, if both
$\architecture_1\preceq \architecture_2$ and $\architecture_2\preceq
\architecture_1$ hold.

We will also need a generalisation of the previous definitions in which we
compare labellings computed by MPNNs at different rounds. This is formalised as
follows.

\begin{definition}\label{def:classg}\normalfont
Consider two MPNNs $M_1$ and $M_2$ with $T_1$ and $T_2$ rounds, respectively.
Let $\pmb{\ell}_{M_1}^{(t)}$ and $\pmb{\ell}_{M_2}^{(t)}$ be their corresponding
labellings on an input graph $( G,\pmb{\nu})$ obtained after $t$ rounds of
computation. Let $g:\mathbb{N}\to \mathbb{N}$ be a monotonic function such that
$g(T_1) = T_2$. We say that $M_1$ is \textit{$g$-weaker} than $M_2$, denoted by
$M_1\preceq_{g} M_2$, if $\pmb{\ell}_{M_2}^{g(t)}\sqsubseteq
\pmb{\ell}_{M_1}^{(t)}$ for every $0 \le t\le T_1$.\hfill$\blacksquare$
\end{definition}

Only the following special cases of this definition, depending on extra
information regarding a function $g:\mathbb{N}\to\mathbb{N}$, will be relevant
in this paper:
\begin{itemize}
    \item $g(t)=t$. This case corresponds to Definition~\ref{def:mpnnweak}. If
   $M_1\preceq_{g} M_2$, then we simply say that $M_1$ is weaker than $M_2$, and
   write $M_1\preceq M_2$, as before. 
   \item $g(t)=t+1$. If $M_1\preceq_{g} M_2$,
   then we say that $M_1$ is weaker than $M_2$ \textit{with one step ahead}. We
   denote this by $M_1\preceq_{+1} M_2$.
   \item $g(t)=ct$ for some constant $c$.
   If $M_1\preceq_{g} M_2$, then we say that $M_1$ is weaker than $M_2$
   \textit{with a linear factor of $c$}. We denote this by $M_1\preceq_{\times
   c} M_2$.
\end{itemize}
We lift these definitions to classes of MPNNs, just like in Definition~\ref{def:classesweak}.

\section{The distinguishing power of anonymous MPNNs}\label{sec:anonymous}
In this section we compare classes of anonymous MPNNs in terms of their
distinguishing power using Definition~\ref{def:classesweak}.

We recall from Section~\ref{sec:MPNNs} that anonymous MPNNs are MPNNs whose message
functions only depend on the previous labels of the vertices involved. The
distinguishing power of anonymous MPNNs (or aMPNNs, for short) is well understood.
Indeed, as we will shortly see, it follows from two independent
works~\cite{xhlj19,grohewl} that the distinguishing power of aMPNNs can be linked to
the distinguishing power of the WL algorithm.

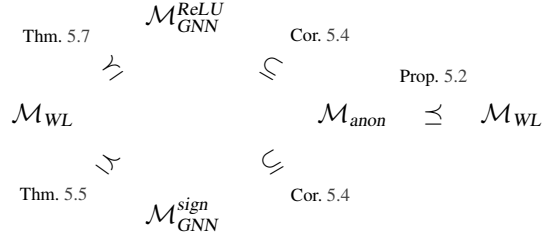
\begin{figure}[t]
    \centering
    \begin{tikzpicture}[node distance=1cm]
        \node (lwl) {$\mathcal{M}_{\textsl{WL}}$};
        \node[above right= of lwl] (relu) {$\mathcal{M}_{\textsl{GNN}}^{\textsl{ReLU}}$};
        \node[below right= of lwl] (sign) {$\mathcal{M}_{\textsl{GNN}}^{\textsl{sign}}$};
        \node[right=3cm of lwl] (anon) {$\mathcal{M}_{\textsl{anon}}$};
        \node[right= of anon] (rwl) {$\mathcal{M}_{\textsl{WL}}$};
        
        \path
        (lwl) edge[draw=none] node[sloped,label=above:\scriptsize{Thm.~\ref{thm:equalstrong}}] {$\preceq$} (relu)
        (lwl) edge[draw=none] node[sloped,label=below:\scriptsize{Thm.~\ref{thm:grohe_lower}}] {$\preceq$} (sign)
        (relu) edge[draw=none] node[sloped,label=above:\scriptsize{Cor.~\ref{corr:GNNwANO}}]{$\subseteq$} (anon)
        (sign) edge[draw=none] node[sloped,label=below:\scriptsize{Cor.~\ref{corr:GNNwANO}}]{$\subseteq$} (anon)
        (anon) edge[draw=none] node[label=above:\scriptsize{Prop.~\ref{pro:eqstrongWL}}] {$\preceq$} (rwl)
        ;
    \end{tikzpicture}
    \caption{Summary of relationships amongst major anonymous MPNN classes considered in Section~\ref{sec:anonymous}.}
    \label{fig:reduxs}
\end{figure}

Let $( G,\pmb{\nu})$ be a labelled graph. We will consider the following classes of
aMPNNs. We denote by $\architectureWL$ the class of aMPNNs consisting of an aMPNN
$M^T_{\textsl{WL}}$, for each $T \in \mathbb{N}$, originating from the WL algorithm
(see Example~\ref{ex:WL}) being ran for $T$ rounds. In a slight abuse of notation,
we will simply write $M_{\textsl{WL}}$ when $T$ is clear from the context. Recall
that the class of anonymous MPNNs is denoted $\architectureano$. Finally, we
introduce two classes of aMPNNs which are of special interest: those arising from
the graph neural networks considered in~\cite{grohewl}. In Example~\ref{ex:GNN} we
established that such graph neural networks correspond to aMPNNs. Let us denote by
$\architecture^\sigma_{\textsl{GNN}}$ the class of aMPNNs with message and update
functions of the form
\begin{equation}\textsc{Msg}^{(t)}\bigl(\mathbf{x},\mathbf{y},-,-):=\mathbf{y}\mathbf{W}_2^{(t)}
\text{ and } 
\textsc{Upd}^{(t)}(\mathbf{x},\mathbf{y}):=\sigma\left(\mathbf{x}\mathbf{W}_1^{(t)}+\mathbf{y} + \mathbf{b}^{(t)}\right) \label{eq:MPNN-GNN}
\end{equation}
for any $\mathbf{x},\mathbf{y}\in\mathbb{A}^{s_{t-1}}$, $\mathbf{W}_1^{(t)}\in\mathbb{A}^{s_{t-1}\times
s_t}$, $\mathbf{W}_2^{(t)}\in\mathbb{A}^{s_{t-1}\times s_t}$, bias vector
$\mathbf{b}^{(t)}\in\mathbb{A}^{s_t}$, and non-linear activation function $\sigma$.

The following is our main result for this section.
\begin{theorem}\label{thm:anonymous}
The classes $\mathcal{M}_{\textsl{WL}}$, $\mathcal{M}_{\textsl{GNN}}^{\textsl{ReLU}}$, $\mathcal{M}_{\textsl{GNN}}^{\textsl{sign}}$ and $\mathcal{M}_{\textsl{anon}}$ are
all equally strong.
\end{theorem}
We prove this theorem in the following subsections by providing the relationships that are summarised in Figure~\ref{fig:reduxs}.

\subsection{General anonymous MPNNs}
We presently focus on the relation between the WL algorithm and anonymous MPNNs in
general. More specifically, we establish that these are equally strong. We remark
that in the proof of Theorem~\ref{thm:anonymous} we only need that
$\architectureano$ is weaker than $\architectureWL$, as is indicated in
Figure~\ref{fig:reduxs}.

\begin{proposition}[Based on~\cite{xhlj19,grohewl}]\label{pro:eqstrongWL}
The classes $\architectureano$ and  $\architectureWL$ are equally strong.
\end{proposition}
\begin{proof}
First, we prove that $\architectureWL$ is weaker than $\architectureano$. It
suffices to note that $\mathcal{M}_{\textsl{WL}} \subseteq \architectureano$.

It remains to argue that $\architectureano$ is weaker than $\architectureWL$. The
proof is a trivial adaptation of the proofs of Lemma 2 in~\cite{xhlj19} and Theorem
5 in~\cite{grohewl}. We show, by induction on the number of rounds of computation,
that $\pmb{\ell}_{M_{\textsl{WL}}}^{(t)}\sqsubseteq \pmb{\ell}_M^{(t)}$ for all $M
\in \architectureano$ and every $t\geq 0$.

Clearly, this holds for $t=0$ since
$\pmb{\ell}_{M_{\textsl{WL}}}^{(0)}=\pmb{\ell}_M^{(0)}:=\pmb{\nu}$, by definition.
We assume next that the induction hypothesis holds up to round $t-1$ and consider
round $t$. Let $v$ and $w$ be two vertices such that
$(\pmb{\ell}_{M_{\textsl{WL}}}^{(t)})_v=(\pmb{\ell}_{M_{\textsl{WL}}}^{(t)})_w$
holds. This implies, by the definition of $M_{\textsl{WL}}$, that
$(\pmb{\ell}_{M_{\textsl{WL}}}^{(t-1)})_v=(\pmb{\ell}_{M_{\textsl{WL}}}^{(t-1)})_w$
and
$$
\ldbl (\pmb{\ell}_{M_{\textsl{WL}}}^{(t-1)})_u\mid u\in N_G(v) \rdbl=
\ldbl (\pmb{\ell}_{M_{\textsl{WL}}}^{(t-1)})_u\mid u\in N_G(w) \rdbl.
$$
By the induction hypothesis, this implies that
$(\pmb{\ell}_{M}^{(t-1)})_v=(\pmb{\ell}_{M}^{(t-1)})_w$ and
$$
\ldbl (\pmb{\ell}_{M}^{(t-1)})_u\mid u\in N_G(v) \rdbl=
\ldbl (\pmb{\ell}_{M}^{(t-1)})_u\mid u\in N_G(w) \rdbl.
$$
As a consequence, there is a bijection between $N_G(v)$ and $N_G(w)$ such that to
every vertex $u\in N_G(v)$ we can assign a unique vertex $u'\in N_G(w)$ such that
$(\pmb{\ell}_{M}^{(t-1)})_u=(\pmb{\ell}_{M}^{(t-1)})_{u'}$. Hence,
$$
\textsc{Msg}^{(t)}\left((\pmb{\ell}_{M}^{(t-1)})_v,(\pmb{\ell}_{M}^{(t-1)})_u,-,-\right)=
\textsc{Msg}^{(t)}\left((\pmb{\ell}_{M}^{(t-1)})_w,(\pmb{\ell}_{M}^{(t-1)})_{u'},-,-\right).
$$
Since this mapping between $N_G(v)$ and $N_G(w)$ is a bijection we also have:
$$
\mathbf{m}^{(t)}_v=\sum_{u\in N_G(v)}\textsc{Msg}^{(t)}\left((\pmb{\ell}_{M}^{(t-1)})_v,(\pmb{\ell}_{M}^{(t-1)})_u,-,-\right)=\sum_{u'\in N_G(w)}\textsc{Msg}^{(t)}\left((\pmb{\ell}_{M}^{(t-1)})_w,(\pmb{\ell}_{M}^{(t-1)})_{u'},-,-\right)=\mathbf{m}^{(t)}_w.
$$
We may thus conclude that $$(\pmb{\ell}_{M}^{(t)})_v=\textsc{Upd}^{(t)}\left((\pmb{\ell}_{M}^{(t-1)})_v,\mathbf{m}^{(t)}_v\right)=\textsc{Upd}^{(t)}\left((\pmb{\ell}_{M}^{(t-1)})_w,\mathbf{m}^{(t)}_w\right)=(\pmb{\ell}_{M}^{(t)})_w,
$$
as desired.
\end{proof}

We remark that we cannot use the results in~\cite{xhlj19} and~\cite{grohewl} as a
black box because the class $\architectureano$ is more general than the class
considered in those papers. The proofs in \cite{xhlj19} and \cite{grohewl} relate to
graph neural networks which, in round $t\geq 1$, compute for each vertex $v$ a label
$\pmb{\ell}^{(t)}_{v}$, as follows:
\begin{equation}
\pmb{\ell}^{(t)}_{v}:=
f_{\textsl{comb}}^{(t)}\left(
\pmb{\ell}_{v}^{(t-1)},f_{\textsl{aggr}}^{(t)}\left(\ldbl \pmb{\ell}^{(t-1)}_{u} \mid u \in N_G(v) \rdbl\right)
\right), \label{eq:combaggr}
\end{equation}
where $f_{\textsl{comb}}^{(t)}$ and $f_{\textsl{aggr}}^{(t)}$ are general
(computable) combination and aggregation functions which we assume to assign labels
in $\mathbb{A}^{s_t}$. Furthermore, $\pmb{\ell}^{(0)}:=\pmb{\nu}$, just as before.
Every graph neural network of the form~(\ref{eq:combaggr}) is readily cast as an
aMPNN. Indeed, it suffices to observe, just as we did in Example~\ref{ex:WL}, that
the aggregation functions $f_{\textsl{aggr}}^{(t)}\bigl(\ldbl \pmb{\ell}^{(t-1)}_{u}
\mid u \in N_G(v) \rdbl\bigr)$ can be written in the form $g^{(t)}\bigl(\sum_{u\in
N_G(v)} h^{(t)}(\pmb{\ell}^{(t-1)}_{u})\bigr)$, based on Lemma 5 from~\cite{xhlj19}.

Suppose that $\pmb{\nu}:V\to\mathbb{A}^{s_0}$. It now suffices to define for every
$t \geq 1$, every $\mathbf{x}$ and $\mathbf{y}$ in $\mathbb{A}^{s_{t-1}}$, every
$v\in V$ and $u\in N_G(u)$:
\begin{equation}
\textsc{Msg}^{(t)}(\mathbf{x},\mathbf{y},-,-):=h^{(t)}(\mathbf{y}) \text{ and } \textsc{Upd}^{(t)}(\mathbf{x},\mathbf{y}):=f_{\textsl{comb}}^{(t)}\left(\mathbf{x},g^{(t)}\left(\mathbf{y}\right)\right).\label{eq:combaggrtoaMPNN}
\end{equation}
This is clearly an aMPNN which computes the same labelling as~(\ref{eq:combaggr}).

The aMPNNs that we consider in this paper are slightly more general than those
defined by (\ref{eq:combaggrtoaMPNN}). Indeed, we consider message functions that
can also depend on the previous label $\pmb{\ell}_v^{(t-1)}$. In contrast, the
message functions in~(\ref{eq:combaggrtoaMPNN}) only depend on $\mathbf{y}$, which
corresponds to the previous labels $\pmb{\ell}_u^{(t-1)}$ of neighbours $u\in
N_G(v)$. Let $\architecture{}^{-}_{\textsl{anon}}$ denote the class of aMPNNs whose
message functions only depend on the previous labels of neighbours. It now suffices
to observe (see Example~\ref{ex:WL}) that $M_{\textsl{WL}}\in
\architecture{}^{-}_{\textsl{anon}}$ to infer, combined with
Proposition~\ref{pro:eqstrongWL}, that:
 \begin{corollary}
	 The classes $\architecture{}^{-}_{\textsl{anon}}$, $\architectureano$ and $\architectureWL$ are all equally strong.
 \end{corollary}
We observe, however, that this does not imply that for every aMPNN $M$ in
$\architectureano$ there exists an aMPNN $M'$ in
$\architecture{}^{-}_{\textsl{anon}}$ such that $\pmb{\ell}_{M}^{(t)}\equiv
\pmb{\ell}_{M'}^{(t)}$ for all $t\geq 0$. Indeed, the corollary implies that for
every $M$ in $\architectureano$ there exists an aMPNN $M'$ in
$\architecture{}^{-}_{\textsl{anon}}$ such that $M\preceq M'$, and there exists an
$M''$ in $\architectureano$, possibly different from $M$, such that $M'\preceq M''$.
In fact, such an aMPNN $M''$, in this case is $M_{\textsl{WL}}$.

\subsection{Graph neural network-based anonymous MPNNs}
In this subsection we study the subclasses of aMPNNs arising from graph neural
network architectures. For convenience, let us write $\architecture_{\textsl{GNN}}
:= \architecture^{\textsl{sign}}_{\textsl{GNN}} \cup
\architecture^{\textsl{ReLU}}_{\textsl{GNN}}$.

We start by stating a direct consequence of Proposition~\ref{pro:eqstrongWL}. It
follows by observing that $\architecture_{\textsl{GNN}}$ is a subclass of
$\architectureano$ as presented in Example~\ref{ex:GNN}.
\begin{corollary}\label{corr:GNNwANO}
	The class 
$\architecture_{\textsl{GNN}}$ is weaker than $\architectureano$ and is thus also weaker than $\architectureWL$.
\end{corollary}

More challenging is to show that $\architecture^{\textsl{sign}}_{\textsl{GNN}}$,
$\architecture^{\textsl{ReLU}}_{\textsl{GNN}}$ and $\architectureWL$, and thus also
$\architectureano$, are equally strong. The following results are known.

\begin{theorem}[\cite{grohewl}] \label{thm:grohe_lower}
(i)~The classes $\architecture_{\textsl{GNN}}^{\textsl{sign}}$ and $\architectureWL$
are equally strong. (ii)~The class $\architecture_{\textsl{GNN}}^{\textsl{ReLU}}$ is
weaker than $\architectureWL$, and $\architectureWL$ is weaker than
$\architecture_{\textsl{GNN}}^{\textsl{ReLU}}$, with a factor of two, i.e.,
$\architectureWL\preceq_{\times 2}\architecture_{\textsl{GNN}}^{\textsl{ReLU}}$.
\end{theorem}
The reason for the factor of two in (ii) in Theorem~\ref{thm:grohe_lower} is due to
a simulation of the sign activation function by means of a two-fold application of
the ReLU function. We next show that this factor of two can be avoided. As a side
effect, we obtain a simpler aMPNN $M$ in $\architecture_{\textsl{GNN}}$, satisfying
$M _{\textsl{WL}} \preceq M$, than the one constructed in \cite{grohewl}. The proof
strategy is inspired by that of~\cite{grohewl}. Crucial in the proof is the notion
of row-independence modulo equality, which we define next.

\begin{definition}[Row-independence modulo equality]\label{def:label2}\
A labelling $\pmb{\ell}:V\to\mathbb{A}^s$ is \textit{row-independent modulo
equality} if the set of unique labels assigned by $\pmb{\ell}$ is linearly
independent. \hfill$\blacksquare$
\end{definition}

In what follows, we always assume that the initial labelling $\pmb{\nu}$ of $G$ is
row-independent modulo equality. One can always ensure this by extending the labels.

\begin{theorem}\label{thm:equalstrong}
The classes $\architecture_{\textsl{GNN}}^{\textsl{ReLU}}$ and $\architectureWL$ are
equally strong.
\end{theorem}
\begin{proof}
We already know that $\architecture_{\textsl{GNN}}^{\textsl{ReLU}}$ is weaker than
$\architectureWL$ (Theorem~\ref{thm:grohe_lower} and also
Corollary~\ref{corr:GNNwANO}). It remains to show that $\architectureWL$ is weaker
than $\architecture_{\textsl{GNN}}^{\textsl{ReLU}}$. That is, given an aMPNN
$M_{\textsl{WL}}$, we need to construct an aMPNN $M$ in
$\architecture_{\textsl{GNN}}^{\textsl{ReLU}}$ such that
$\pmb{\ell}_{M}^{(t)}\sqsubseteq \pmb{\ell}_{M_{\textsl{WL}}}^{(t)}$, for all $t\geq
0$. We observe that since $\pmb{\ell}_{M_{\textsl{WL}}}^{(t)}\sqsubseteq
\pmb{\ell}_{M}^{(t)}$ for any $M$ in $\architecture_{\textsl{GNN}}^{\textsl{ReLU}}$,
this is equivalent to constructing an $M$ such that $\pmb{\ell}_{M}^{(t)}\equiv
\pmb{\ell}_{M_{\textsl{WL}}}^{(t)}$.

The proof is by induction on the number of computation rounds. The aMPNN $M$ in
$\architecture_{\textsl{GNN}}^{\textsl{ReLU}}$ that we will construct will use
message and update functions of the form:
\begin{equation}\textsc{Msg}^{(t)}\bigl(\mathbf{x},\mathbf{y},-,-):=\mathbf{y}\mathbf{W}^{(t)}
\text{ and } 
\textsc{Upd}^{(t)}(\mathbf{x},\mathbf{y}):=\text{ReLU}\left(p\mathbf{x}\mathbf{W}^{(t)}+\mathbf{y} + \mathbf{b}^{(t)}\right)
\end{equation}
for some value $p\in\mathbb{A}$, $0<p<1$, weight matrix
$\mathbf{W}^{(t)}\in\mathbb{A}^{s_{t-1}\times s_t}$, and bias vector
$\mathbf{b}^{(t)}\in\mathbb{A}^{s_t}$. Note that, in contrast to aMPNNs of the
form~(\ref{eq:MPNN-GNN}), we only have one weight matrix per round, instead of two,
at the cost of introducing an extra parameter $p\in\mathbb{A}$. Furthermore, the
aMPNN constructed in \cite{grohewl} uses two distinct weight matrices in
$\mathbb{A}^{(s_{t-1} + s_0)\times (s_t + s_0)}$ (we come back to this at the end of
this section) whereas our weight matrices are elements of $\mathbb{A}^{s_{t-1}\times
s_t}$ and thus of smaller dimension.

The induction hypothesis is that $\pmb{\ell}^{(t)}_M\equiv
\pmb{\ell}_{M_{\textsl{WL}}}^{(t)}$ and that $\pmb{\ell}^{(t)}_M$ is row-independent
modulo equality.

For $t=0$, we have that for any $M\in \architecture_{\textsl{GNN}}^{\textsl{ReLU}}$,
$\pmb{\ell}_M^{(0)}=\pmb{\ell}_{M_{\textsl{WL}}}^{(0)}:=\pmb{\nu}$, by definition.
Moreover, $\pmb{\ell}_M^{(0)}$ is row-independent modulo equality because
$\pmb{\nu}$ is so, by assumption.

We next assume that up to round $t-1$ we have found weight matrices and bias vectors
for $M$ such that $\pmb{\ell}_M^{(t-1)}$ satisfies the induction hypothesis. We will
show that for round $t$ we can find a weight matrix
$\mathbf{W}^{(t)}\in\mathbb{A}^{s_{t-1}\times s_t }$ and bias vector
$\mathbf{b}^{(t)}\in\mathbb{A}^{s_t}$ such that $\pmb{\ell}_M^{(t)}$ also satisfies
the hypothesis.

Let $\mathbf{L}^{(t-1)}\in\mathbb{A}^{n\times s_{t-1}}$ denote the matrix consisting
of rows $(\pmb{\ell}_M^{(t-1)})_v$, for $v\in V$. Moreover, we denote by
$\mathsf{uniq}(\mathbf{L}^{(t-1)})$ a $(m\times s_{t-1})$-matrix consisting of the
$m$ unique rows in $\mathbf{L}^{(t-1)}$ (the order of rows is irrelevant). We denote
the rows in $\mathsf{uniq}(\mathbf{L}^{(t-1)})$ by
$\mathbf{a}_1,\ldots,\mathbf{a}_m\in\mathbb{A}^{s_{t-1}}$. By the induction
hypothesis, these rows are linearly independent. Following the same argument as
in~\cite{grohewl}, this implies that there exists an $(s_{t-1}\times m)$-matrix
$\mathbf{U}^{(t)}$ such that
$\mathsf{uniq}(\mathbf{L}^{(t-1)})\mathbf{U}^{(t)}=\mathbf{I}$. Let us denote by
$\mathbf{e}_1,\ldots,\mathbf{e}_m\in\mathbb{A}^m$ the rows of $\mathbf{I}$. In other
words, in $\mathbf{e}_i$, all entries are zero except for entry $i$, which holds
value $1$.

We consider the following intermediate labelling $\pmb{\mu}^{(t)}:V\to\mathbb{A}^m$
defined by
\begin{equation}
v\mapsto \left((\mathbf{A}+p\mathbf{I})\mathbf{L}^{(t-1)}\mathbf{U}^{(t)}\right)_{v}.\label{eq:labelmu}
\end{equation}
We know that for every vertex $v$, $(\pmb{\ell}_M^{(t-1)})_v$ corresponds to a
unique row $\mathbf{a}_i$ in $\mathsf{uniq}(\mathbf{L}^{(t-1)})$. We denote the
index of this row by $\rho(v)$. More specifically,
$(\pmb{\ell}_M^{(t-1)})_v=\mathbf{a}_{\rho(v)}$. Let $N_G(v,i):=\{u \st u\in N_G(v),
\rho(v)=i\}$. That is, $N_G(v,i)$ consists of all neighbours $u$ of $v$ which are
labelled as $\mathbf{a}_i$ by $\pmb{\ell}_M^{(t-1)}$. It is now readily verified
that the label $\pmb{\mu}^{(t)}_v$ defined in~(\ref{eq:labelmu}) is of the form
\begin{equation}
\pmb{\mu}^{(t)}_v=
p\mathbf{e}_{\rho(v)} + \sum_{i=1}^m |N_G(v,i)|\mathbf{e}_i.  \label{eq:linearcomb}
\end{equation}
We clearly have that $\pmb{\ell}_{M_{\textsl{WL}}}^{(t)}\sqsubseteq\pmb{\mu}^{(t)}$. The converse also holds, as is shown in the following lemma.
\begin{lemma}
For any two vertices $v$ and $w$, we have that 
	$\pmb{\mu}^{(t)}_v=\pmb{\mu}^{(t)}_w$ implies 
	$(\pmb{\ell}_{M_{\textsl{WL}}}^{(t)})_v=(\pmb{\ell}_{M_{\textsl{WL}}}^{(t)})_w$.
\end{lemma}
\begin{proof}
We argue by contradiction. Suppose, for the sake of contradiction, that there exist two vertices $v,w\in V$ such that 
\begin{equation}
		\pmb{\mu}^{(t)}_{v}=\pmb{\mu}^{(t)}_{w} \text{ and } (\pmb{\ell}_{M_{\textsl{WL}}}^{(t)})_v\neq(\pmb{\ell}_{M_{\textsl{WL}}}^{(t)})_w \label{eq:contra-pwl}
\end{equation}
hold. We show that this is impossible for any value $p$ satisfying $0<p<1$.
(Recall from~\eqref{eq:linearcomb} that $\pmb{\mu}^{(t)}_v$ depends on $p$.)

We distinguish between the following two cases. If $ (\pmb{\ell}_{M_{\textsl{WL}}}^{(t)})_v\neq(\pmb{\ell}_{M_{\textsl{WL}}}^{(t)})_w$ then either
\begin{enumerate}[(i)]
\item \label{itm:case1} $(\pmb{\ell}_{M_{\textsl{WL}}}^{(t-1)})_v\neq(\pmb{\ell}_{M_{\textsl{WL}}}^{(t-1)})_w$; or 
\item \label{itm:case2}
$(\pmb{\ell}_{M_{\textsl{WL}}}^{(t-1)})_v=(\pmb{\ell}_{M_{\textsl{WL}}}^{(t-1)})_w$
and
	$
	\ldbl (\pmb{\ell}_{M_{\textsl{WL}}}^{(t-1)})_u \st u \in N_G(v) \rdbl\neq
	\ldbl(\pmb{\ell}_{M_{\textsl{WL}}}^{(t-1)})_u \st u \in N_G(w) \rdbl.
	$
\end{enumerate}

We first consider case~\ref{itm:case1}. Observe that
$(\pmb{\ell}_{M_{\textsl{WL}}}^{(t-1)})_v\neq(\pmb{\ell}_{M_{\textsl{WL}}}^{(t-1)})_w
$ implies that $(\pmb{\ell}^{(t-1)}_M)_{v}\neq (\pmb{\ell}_M^{(t-1)})_w$. This
follows from the induction hypothesis $\pmb{\ell}^{(t-1)}_M\equiv
\pmb{\ell}_{M_{\textsl{WL}}}^{(t-1)}$. It now suffices to observe that
$\pmb{\mu}^{(t)}_{v}=\pmb{\mu}^{(t)}_{w}$ implies that the corresponding linear
combinations, as described in~\eqref{eq:linearcomb}, satisfy:
$$
 p\mathbf{e}_{\rho(v)} + \sum_{i=1}^m |N_G(v,i)|\mathbf{e}_i =
 p\mathbf{e}_{\rho(w)} + 
\sum_{i=1}^m |N_G(w,i)|\mathbf{e}_i.
$$
We can assume, without loss of generality, that
$(\pmb{\ell}^{(t-1)}_M)_{v}=\mathbf{a}_1$ and
$(\pmb{\ell}^{(t-1)}_M)_{w}=\mathbf{a}_2$. Recall that $\mathbf{a}_1$ and
$\mathbf{a}_2$ are two distinct labels. Then, the previous equality implies:
\[
\left(|N_G(v,1)|+p-|N_G(w,1)|\right)\mathbf{e}_1+
\left(|N_G(v,2)|-|N_G(w,2)|-p\right)\mathbf{e}_2 +
\sum_{i=3}^m \left(|N_G(v,i)|-|N_G(w,i)|\right)\mathbf{e}_i=0.
\]
Since $\mathbf{e}_1,\ldots,\mathbf{e}_m$ are linearly independent, this implies that
$|N_G(v,i)|-|N_G(w,i)|=0$ for all $i=3,\ldots,m$ and $|N_G(v,1)|+p-|N_G(w,1)|=0$ and
$|N_G(v,2)|-|N_G(w,2)|-p=0$. Since $|N_G(v,1)|-|N_G(w,1)|\in\mathbb{Z}$ and $0<p<1$,
this is impossible. We may thus conclude that case~\ref{itm:case1} cannot occur.

Suppose next that we are in case~\ref{itm:case2}. Recall that for
case~\ref{itm:case2}, we have that
$(\pmb{\ell}{}_{M_{\textsl{WL}}}^{(t-1)})_v=(\pmb{\ell}{}_{M_{\textsl{WL}}}^{(t-1)})_
w$ and thus also $(\pmb{\ell}{}_M^{(t-1)})_v=(\pmb{\ell}{}_M^{(t-1)})_w$. Using the
same notation as above, we may assume that
$(\pmb{\ell}{}_M^{(t-1)})_v=(\pmb{\ell}{}_M^{(t-1)})_w=\mathbf{a}_1$. In
case~\ref{itm:case2}, however, we have that
	$
	\ldbl (\pmb{\ell}{}_{M_{\textsl{WL}}}^{(t-1)})_{u} \st u \in N_G(v) \rdbl\neq
	\ldbl (\pmb{\ell}{}_{M_{\textsl{WL}}}^{(t-1)})_{u} \st u \in N_G(w) \rdbl
	$ and thus also 
	$
	\ldbl (\pmb{\ell}{}_{M}^{(t-1)})_{u} \st u \in N_G(v) \rdbl\neq
	\ldbl (\pmb{\ell}{}_{M}^{(t-1)})_{u} \st u \in N_G(w) \rdbl
	$.
That is, there must exist a label  assigned by $\pmb{\ell}{}_M^{(t-1)}$ that does not
occur the same number of times in the neighbourhoods of $v$ and $w$, respectively.
Suppose that this label is $\mathbf{a}_2$. The case when this label is
$\mathbf{a}_1$ can be treated similarly. It now suffices to observe that
$\pmb{\mu}^{(t)}_{v}=\pmb{\mu}^{(t)}_{w}$ implies that the corresponding linear
combinations, as described in~(\ref{eq:linearcomb}), satisfy:
$$
\left(|N_G(v,1)|+p\right)\mathbf{e}_1 +|N_G(v,2)|\mathbf{e}_2+\sum_{i=3}^m |N_G(v,i)|\mathbf{e}_i=
\left(|N_G(w,1)|+p\right)\mathbf{e}_1 +|N_G(w,2)|\mathbf{e}_2+\sum_{i=3}^m |N_G(w,i)|\mathbf{e}_i.
$$
Using a similar argument as before, based on the linear independence of
$\mathbf{e}_1,\ldots,\mathbf{e}_m$, we can infer that $|N_G(v,2)|=|N_G(w,2)|$. We
note, however, that $\mathbf{a}_2$ appeared a different number of times among the
neighbours of $v$ and $w$. Hence, also case~\ref{itm:case2} is ruled out and our
assumption~\eqref{eq:contra-pwl} is invalid. This implies
$\pmb{\mu}^{(t)}\sqsubseteq\pmb{\ell}_{M_{\textsl{WL}}}^{(t)}$, as desired and thus
concludes the proof of the lemma.
\end{proof}
From here, to continue with the proof of Theorem~\ref{thm:equalstrong}, we still
need to take care of the ReLU activation function. Importantly, its application
should ensure row-independence modulo equality and make sure the labelling
``refines'' $\pmb{\ell}^{(t)}_{M_{\textsl{WL}}}$. To do so, we again follow closely
the proof strategy of~\cite{grohewl}. More specifically, we will need an analogue of
the following result. In the sequel we denote by $\mathbf{J}$ a matrix with all
entries having value $1$ and whose size will be determined from the context.

\begin{lemma}[Lemma 9 from~\cite{grohewl}]\label{lem:signlemma9}
  Let $\mathbf{C}\in \mathbb{A}^{m\times w}$ be a matrix in which all entries are
  non-negative and all rows are pairwise disjoint. Then there exists a matrix
  $\mathbf{X}\in\mathbb{A}^{w\times m}$ such that $\text{\normalfont
  sign}(\mathbf{CX}-\mathbf{J})$ is a non-singular matrix in $\mathbb{A}^{m\times
  m}$.
\end{lemma}

We prove the following for the ReLU function.
\begin{lemma}\label{lem:ReLUlemma9}
  Let $\mathbf{C}\in \mathbb{A}^{m\times w}$ be a matrix in which all entries are
  non-negative, all rows are pairwise disjoint and such that no row consists
  entirely out of zeroes\footnote{Compared to Lemma~\ref{lem:signlemma9}, we
  additionally require non-zero rows.}. Then there exists a matrix
  $\mathbf{X}\in\mathbb{A}^{w\times m}$ and a constant $q\in\mathbb{A}$ such that
  $\text{\normalfont ReLU}(\mathbf{CX}-q\mathbf{J})$ is a non-singular matrix in
  $\mathbb{A}^{m\times m}$.
\end{lemma}
\begin{proof}
Let $C$ be the maximal entry in $\mathbf{C}$ and consider the column vector
$\mathbf{z}=(1,C,C^2,\ldots,C^{w-1})^{\textsc{t}}\in\mathbb{A}^{w\times 1}$. Then
each entry in $\mathbf{c}=\mathbf{C}\mathbf{z}\in\mathbb{A}^{m\times 1}$ is positive
and all entries in $\mathbf{c}$ are pairwise distinct. Let $\mathbf{P}$ be a
permutation matrix in $\mathbb{A}^{m\times m}$ such that
$\mathbf{c}'=\mathbf{P}\mathbf{c}$ is such that
$\mathbf{c}'=(c_1',c_2',\ldots,c_m')^{\textsc{t}}\in\mathbb{A}^{m\times 1}$ with
$c_1'> c_2'>\cdots > c_m'>0$. Consider
$\mathbf{x}=\left(\frac{1}{c_1'},\ldots,\frac{1}{c_m'}\right)\in \mathbb{A}^{1\times
m}$. Then, for $\mathbf{E}=\mathbf{c}'\mathbf{x}\in\mathbb{A}^{m\times m}$
$$
\mathbf{E}_{ij}=\frac{c_i'}{c_j'}  \text{ and } \mathbf{E}_{ij}=\begin{cases}  1 & \text{if $i=j$}\\
>1 & \text{if $i<j$}\\
< 1 & \text{if $i>j$}.
\end{cases}
$$
Let $q$ be the greatest value in $\mathbf{E}$ smaller than $1$. Consider
$\mathbf{F}=\mathbf{E}- q\mathbf{J}$. Then,
$$
\mathbf{F}_{ij}=\frac{c_i'}{c_j'}- q \text{ and } \mathbf{F}_{ij}=\begin{cases}  1-q & \text{if $i=j$} \\
> 0 & \text{if $i<j$}\\
\leq 0  & \text{if $i>j$}.
\end{cases}
$$
As a consequence,
$$
\text{ReLU}(\mathbf{F})_{ij}=\begin{cases}  1-q & \text{if $i=j$}\\
>0 & \text{if $i<j$}\\
0  & \text{if $i>j$}.
\end{cases}
$$
This is an upper triangular matrix with (nonzero) value $1-q$ on its diagonal. It is
therefore non-singular.

We now observe that
$\mathbf{Q}\text{ReLU}(\mathbf{F})=\text{ReLU}(\mathbf{Q}\mathbf{F})$ for any row
permutation $\mathbf{Q}$. Furthermore, non-singularity is preserved under row
permutations and $\mathbf{Q}\mathbf{J}=\mathbf{J}$. Hence, if we define
$\mathbf{X}=\mathbf{z}\mathbf{x}$ and use the permutation matrix $\mathbf{P}$, then:
\begin{align*}
\mathbf{P}\text{ReLU}(\mathbf{C}\mathbf{X}-q\mathbf{J})&=
\text{ReLU}(\mathbf{P}\mathbf{C}\mathbf{z}\mathbf{x}-q\mathbf{P}\mathbf{J})=\text{ReLU}(\mathbf{E}-q\mathbf{J}) = \text{ReLU}(\mathbf{F}),
\end{align*}
and we have that $\text{ReLU}(\mathbf{C}\mathbf{X}-q\mathbf{J})$ is non-singular, as
desired. This concludes the proof of the lemma.
\end{proof}

We now apply this lemma to the matrix $\mathsf{uniq}(\mathbf{M}^{(t)})$, with
$\mathbf{M}^{(t)}\in\mathbb{A}^{n\times m}$ consisting of the rows
$\pmb{\mu}^{(t)}_v$, for $v\in V$. Inspecting the expression from
Equation~\eqref{eq:linearcomb} for $\pmb{\mu}^{(t)}_v$ we see that each row in
$\mathbf{M}^{(t)}$ holds non-negative values and no row consists entirely out of
zeroes. Let $\mathbf{X}^{(t)}$ and $q^{(t)}$ be the matrix and constant returned by
Lemma~\ref{lem:ReLUlemma9} such that
$\text{ReLU}\left(\mathsf{uniq}(\mathbf{M}^{(t)})\mathbf{X}^{(t)}-q^{(t)}\mathbf{J}\right)$ is an $m\times m$ non-singular matrix. We now define
$$
\pmb{\ell}_M^{(t)}:=\text{ReLU}\left(\mathbf{M}^{(t)}\mathbf{X}^{(t)}-q^{(t)}\mathbf{J}\right).$$
From the non-singularity of
$\text{ReLU}\left(\mathsf{uniq}(\mathbf{M}^{(t)})\mathbf{X}^{(t)}-q^{(t)}\mathbf{J}
\right)$ we can immediately infer that $\pmb{\ell}_M^{(t)}$ is row-independent modulo
equality. It remains to argue that
$\pmb{\ell}_M^{(t)}\equiv\pmb{\ell}_{M_{\textsl{WL}}}^{(t)}$. This now follows from
the fact that $\pmb{\mu}^{(t)}\equiv \pmb{\ell}_{M_{\textsl{WL}}}^{(t)}$ and each of
the $m$ unique labels assigned by $\pmb{\mu}^{(t)}$ uniquely corresponds to a row in
$\mathsf{uniq}(\mathbf{M}^{(t)})$, which in turn can be mapped bijectively to a row
in
$\text{ReLU}\left(\mathsf{uniq}(\mathbf{M}^{(t)})\mathbf{X}^{(t)}-q^{(t)}\mathbf{J}
\right)$. We conclude by observing that the desired weight matrices and bias vector
at round $t$ for $M$ are now given by
$\mathbf{W}^{(t)}:=\mathbf{U}^{(t)}\mathbf{X}^{(t)}$
and $\mathbf{b}^{(t)}:=-q^{(t)}\mathbf{1}$. This concludes the proof of Theorem~\ref{thm:equalstrong}.
\end{proof}

We remark that the previous proof can be used for
$\architecture_{\textsl{GNN}}^{\textsl{sign}}$ as well. One just has to use
Lemma~\ref{lem:signlemma9} instead of Lemma~\ref{lem:ReLUlemma9}. It is interesting
to note that the bias vector for the sign activation function in
Lemma~\ref{lem:signlemma9} is the same for every $t$. A similar statement holds for
the ReLU function. Indeed, we recall that we apply Lemma~\ref{lem:ReLUlemma9} to
$\mathsf{uniq}(\mathbf{M}^{(t)})$. For every $t$, the entries in this matrix are of
the form $i+p$ (which is smaller than $i+1$) or $i$, for $i\in \{1,2,\dots,n\}$.
Hence, for every $t$, the maximal entry (denoted by $C$ in the proof of
Lemma~\ref{lem:ReLUlemma9}) is upper bounded by $n+1$. The value $q^{(t)}$ relates
to the largest possible ratios, smaller than $1$, of elements in the matrix
constructed in Lemma~\ref{lem:ReLUlemma9}.

When the lemma is applied to an $m \times w$ matrix, this ratio is upper bounded by
$\frac{(n+1)^w-1}{(n+1)^w}$. Note that, since the lemma is applied to matrices
arising from $\pmb{\mu}^{(t)}$, $w$ will always be at most $n$. Hence, taking any
$q^{(t)}:=q$ for $\frac{(n+1)^n-1}{(n+1)^n}<q<1$ suffices. We can take $q$ to be
arbitrarily close to $1$, but not $1$ itself.

We can thus strengthen Theorem~\ref{thm:grohe_lower}, as follows. We denote by
$\architecture_{\textsl{GNN}^-}$ the class of aMPNNs using message and update
functions of the form:
\begin{equation}\textsc{Msg}^{(t)}\bigl(\mathbf{x},\mathbf{y},-,-):=\mathbf{y}\mathbf{W}^{(t)}
\text{ and } 
\textsc{Upd}^{(t)}(\mathbf{x},\mathbf{y}):=\sigma\left(p\mathbf{x}\mathbf{W}^{(t)}+\mathbf{y} -q \mathbf{1}\right), \label{eq:GNNWL}
\end{equation}
parameterised with values $p,q\in\mathbb{A}$, $0\leq p,q\leq 1$ and weight matrices
$\mathbf{W}^{(t)}\in\mathbb{A}^{s_{t-1}\times s_t}$, and where $\sigma$ can be
either the sign or ReLU function.
\begin{corollary}\label{cor:pluspstrongwl}
The class $\architecture_{\textsl{GNN}^-}$ is equally strong as $\architecture_{\textsl{GNN}}$ and is equally strong as $\architectureWL$.\qed
\end{corollary}
We remark that the factor two, needed for the ReLU activation function in
Theorem~\ref{thm:grohe_lower}, has been eliminated. Phrased in terms of graph neural
networks, an aMPNN in $\architecture_{\textsl{GNN}^-}$ is of the form
\begin{equation}
\mathbf{L}^{(t)}=\sigma\left((\mathbf{A}+p\mathbf{I})\mathbf{L}^{(t-1)}\mathbf{W}^{(t)}-q\mathbf{J}\right), \label{GNN:plusp}
\end{equation}
and, thus, these suffice to implement the WL algorithm. It would be interesting to
see how graph neural networks defined by~(\ref{GNN:plusp}), with learnable
parameters $p$ and $q$, perform in practice.
In contrast, if one inspects the proof in ~\cite[pg. 14, Appendix]{grohewl}, even
for the sign activation function, the graph neural network given to implement the WL
algorithm has the more complicated form:
$$
\left(\mathbf{L}^{(0)},\mathbf{L}^{(t)}\right):=\sigma\left(\left(\mathbf{L}^{(0)},\mathbf{L}^{(t-1)}\right)\begin{pmatrix}
\mathbf{I} & \mathbf{0}\\
\mathbf{0} & \mathbf{0}\end{pmatrix}
+\mathbf{A}\left(\mathbf{L}^{(0)},\mathbf{L}^{(t-1)}\right)
\begin{pmatrix}
\mathbf{0} & \mathbf{0}\\
\mathbf{0} & \mathbf{W}^{(t)}\end{pmatrix}-
\left(\mathbf{0}, \mathbf{J}\right)
\right).
$$
We thus have obtained a simpler class of aMPNNs, $\architecture_{\textsl{GNN}^-}$,
which is equally strong as $\architectureWL$. We will see in the next section that
the parameter $p$ also plays an important role for degree-aware aMPNNs.

\begin{table}
\hrule
\hspace*{1ex}
 \caption{Various graph neural network formalisms, as reported in e.g.,\cite{kipf-loose,Wu2019,DBLP:journals/corr/abs-1905-03046}, which correspond to degree-aware MPNNs. We implicitly assume the presence of a  bias matrix $\mathbf{B}^{(t)}$ consisting of copies of the same row $\mathbf{b}^{(t)}$.}
    \label{tab:dMPNNs}
    \centering
    \begin{tabular}{ll}
dGNN$_1$: &
 $\mathbf{L}^{(t)}:=\sigma\left(\mathbf{D}^{-1}\mathbf{A}\mathbf{L}^{(t-1)}\mathbf{W}^{(t)}\right)$ \\
dGNN$_2$: &
$\mathbf{L}^{(t)}:=\sigma\left(\mathbf{D}^{\nicefrac{-1}{2}}\mathbf{A}\mathbf{D}^{\nicefrac{-1}{2}}\mathbf{L}^{(t-1)}\mathbf{W}^{(t)}\right)$\\
dGNN$_3$: &
$
\mathbf{L}^{(t)}:=\sigma\left((\mathbf{D}+\mathbf{I})^{-1}(\mathbf{A}+\mathbf{I})\mathbf{L}^{(t-1)}\mathbf{W}^{(t)}\right)$ \\
dGNN$_4$: &
$
\mathbf{L}^{(t)}:=\sigma\Bigl(\bigl(\mathbf{D}+\mathbf{I}\bigr)^{\nicefrac{-1}{2}} (\mathbf{A}+\mathbf{I})\bigl(\mathbf{D}+\mathbf{I}\bigr)^{\nicefrac{-1}{2}} \mathbf{L}^{(t-1)}\mathbf{W}^{(t)}\Bigr)
$ \\
dGNN$_5$: &
$
\mathbf{L}^{(t)}:=\sigma\left((\mathbf{D}^{\nicefrac{-1}{2}}\mathbf{A}\mathbf{D}^{\nicefrac{-1}{2}}+\mathbf{I})\mathbf{L}^{(t-1)}\mathbf{W}^{(t)}\right)$\\
dGNN$_6$: &
$\mathbf{L}^{(t)}:=\sigma\left((r\mathbf{I}+(1-r)\mathbf{D})^{\nicefrac{-1}{2}}(\mathbf{A}+p\mathbf{I})(r\mathbf{I}+(1-r)\mathbf{D})^{\nicefrac{-1}{2}}\mathbf{L}^{(t-1)}\mathbf{W}^{(t)}\right)$ 
    \end{tabular}
\hspace*{1ex}
\hrule
\end{table}

\section{The distinguishing power of degree-aware MPNNs}\label{sec:dMPNNs}
In this section we compare various classes of degree-aware MPNNs in terms of their
distinguishing power. We recall that degree-aware MPNNs (dMPNNs for short) have
message functions that depend on the labels and degrees of vertices. To compare
these classes we use Definition~\ref{def:classesweak} and also
Definition~\ref{def:classg}. In the latter definition we will be interested in the
function $g(n) = n+1$. That is, when comparing classes of dMPNNs we consider the
notions of being weaker or stronger with $1$ step ahead.

We will also compare degree-aware MPNNs with anonymous MPNNs. Recall that by
Theorem~\ref{thm:anonymous} all classes of anonymous MPNNs considered in
Section~\ref{sec:anonymous} are equivalent for $\equiv$. In particular, they are all
equivalent to the class $\mathcal{M}_{\textsl{WL}}$. Therefore, instead of comparing
a class $\mathcal{M}$ of dMPNNs with all classes considered in
Section~\ref{sec:anonymous} it suffices to compare it with
$\mathcal{M}_{\textsl{WL}}$. For example, if $\architectureWL \preceq_g
\architecture$ then the same relationship to $\architecture$ holds for all classes
in Section~\ref{sec:anonymous} that are equivalent to $\architectureWL$. Similarly,
for when $\architecture \preceq_g \architectureWL$ holds.

Quintessential examples of degree-aware MPNNs are the popular graph convolutional
networks, as introduced by~\cite{kipf-loose}. These are of the form:
$$
\mathbf{L}^{(t)}:=\sigma\Bigl(\bigl(\mathbf{D}+\mathbf{I}\bigr)^{-1/2} (\mathbf{A}+\mathbf{I})\bigl(\mathbf{D}+\mathbf{I}\bigr)^{-1/2} \mathbf{L}^{(t-1)}\mathbf{W}^{(t)}\Bigr),
$$
as already described and phrased as dMPNNs in Example~\ref{ex:KipfasMPNN}. In fact,
many commonly used graph neural networks use degree information. We list a couple of
such formalisms in Table~\ref{tab:dMPNNs}. It is easily verified that these can all
be cast as dMPNNs along the same lines as Example~\ref{ex:KipfasMPNN}. We detail
this later in this section.

We consider the following classes of dMPNNs. First, we recall that
$\architecture_{\textsl{deg}}$ is the class of degree-aware MPNNs. Furthermore, for
$i\in \{1,2,\dots,6\}$, we define $\architecture_{\textsl{dGNN}_i}$ as the class of
dMPNNs originating from a GNN of the form dGNN$_i$, from Table~\ref{tab:dMPNNs}, by
varying the weight matrices $\mathbf{W}^{(t)}$ and, when applicable, the bias
$\mathbf{B}^{(t)}$ and parameters $p$, $r$. The following is our main result for
this section.
\begin{theorem}\label{thm:dmpnn}
For the class of degree-aware MPNNs:
\begin{enumerate}
 \item $\architectureWL \preceq \architecture_{\textsl{deg}}$ and $\architecture_{\textsl{deg}} \not \preceq \architectureWL$;
 \item $\architecture_{\textsl{deg}} \preceq_{+1} \architectureWL$.
\end{enumerate}
For the architectures from Table~\ref{tab:dMPNNs}:
\begin{enumerate}
 \item[3.] $\architecture_{\textsl{dGNN}_i} \not \preceq \architectureWL$ for $i = 2,4,5,6$ and $\architecture_{\textsl{dGNN}_i} \preceq \architectureWL$ for $i = 1,3$;
 \item[4.] $\architectureWL \not \preceq \architecture_{\textsl{dGNN}_i}$ for $1 \le i \le 5$ and $\architectureWL \preceq \architecture_{\textsl{dGNN}_6}$.
\end{enumerate}
\end{theorem}
We prove this theorem in the following subsections by providing the relationships that are summarised in Figure~\ref{fig:dmpnn}. 

\begin{figure}[t]
    \centering
    \begin{tikzpicture}[node distance=0.9cm]
        \node (lwl) {$\mathcal{M}_{\textsl{WL}}$};
        \node[right= of lwl] (deg) {$\architecture_{\textsl{deg}}$};
        \node[above right=1.2cm of deg] (rwl) {$\mathcal{M}_{\textsl{WL}}$};
        \node[below right=1.2cm of deg] (rwl2) {$\architecture_{\textsl{WL}}$};        
               
        \node[right = 3cm of deg] (blwl) {$\mathcal{M}_{\textsl{WL}}$};
        \node[above right= 1.5cm and 1.75cm of blwl] (gnn1) {$\mathcal{M}_{\textsl{dGNN}_1}, \mathcal{M}_{\textsl{dGNN}_3}$};
        \node[right= 2.3cm of blwl] (gnn6) {$\mathcal{M}_{\textsl{dGNN}_6}$};
        \node[below right= 1.5cm and 1cm of blwl] (gnn2) {$\mathcal{M}_{\textsl{dGNN}_2},\mathcal{M}_{\textsl{dGNN}_4},\mathcal{M}_{\textsl{dGNN}_5}$};
        \node[above right= 1.5cm and 1cm of gnn2] (brwl) {$\mathcal{M}_{\textsl{WL}}$};
               
        \path
        (lwl) edge[draw=none] node[label=above:\scriptsize{ Prop.~\ref{prop:notweaker}}] {$\preceq$} (deg)
        (deg) edge[draw=none] node[rotate=45,align=center]{\scriptsize{ Prop.~\ref{prop:notweaker}} \\ $\not \preceq$} (rwl)
        (deg) edge[draw=none] node[rotate=-45,align=center]{\scriptsize{ Prop.~\ref{prop:onestep}} \\ $\preceq_{+1}$} (rwl2)
        (blwl) edge[draw=none] node[above,rotate=45,align=center] {\scriptsize{ Prop.~\ref*{prop:notasstrong}} \\ $\not \preceq$} (gnn1)
        (blwl) edge[draw=none] node[label=above:\scriptsize{ Prop.~\ref{prop:indeed-wl-power}}] {$\preceq$} (gnn6)
        (blwl) edge[draw=none] node[below,rotate=-45,align=center] {\scriptsize{ Prop.~\ref*{prop:notasstrong}} \\ $\not \preceq$} (gnn2)
        (gnn1) edge[draw=none] node[above,rotate=-45,align=center] {\scriptsize{ Cor.~\ref*{ex:landr}} \\ $\preceq$} (brwl)
        (gnn6) edge[draw=none] node[label=above:\scriptsize{Prop.~\ref{prop:notweaker}}] {$\not \preceq$} (brwl)
        (gnn2) edge[draw=none] node[below,rotate=45,align=center] {\scriptsize{ Prop.~\ref*{prop:notweaker}} \\ $\not \preceq$} (brwl)
        ;
    \end{tikzpicture}
\caption{Summary of results comparing degree-aware MPNNs in Theorem~\ref{thm:dmpnn}. We note that Proposition~\ref{prop:notweaker} shows only $\mathcal{M}_{\textsl{dGNN}_4} \not \preceq \mathcal{M}_{\textsl{WL}}$, but $\mathcal{M}_{\textsl{dGNN}_2},\mathcal{M}_{\textsl{dGNN}_5},\mathcal{M}_{\textsl{dGNN}_6} \not \preceq \mathcal{M}_{\textsl{WL}}$ can be easily inferred from it.}\label{fig:dmpnn}
\end{figure}

\subsection{General degree-aware MPNNs}
We first focus on the relation between the WL
algorithm and dMPNNs in general. More specifically, we start with the first item in Theorem~\ref{thm:dmpnn}.
As part of the proof we show that $\mathcal{M}_{\textsl{dGNN}_4} \not \preceq \mathcal{M}_{\textsl{WL}}$. We can similarly show that 
$\mathcal{M}_{\textsl{dGNN}_2},\mathcal{M}_{\textsl{dGNN}_5},\mathcal{M}_{\textsl{dGNN}_6} \not \preceq \mathcal{M}_{\textsl{WL}}$, hereby also settling the first part of the third item in Theorem~\ref{thm:dmpnn}.

\begin{proposition}\label{prop:notweaker}
The class $\architectureWL$ is weaker than $\architecture_{\textsl{deg}}$; but the
class $\architecture_{\textsl{deg}}$ is not weaker than $\architectureWL$.
\end{proposition}
\begin{proof}
To prove the first part of the claim notice that $\architectureano$ is weaker than
$\architecture_{\textsl{deg}}$, simply because any aMPNN is a dMPNN. Then the result
follows from Theorem~\ref{thm:anonymous}.

For the second part it suffices to provide a dMPNN $M$ and a labelled graph $(
G,\pmb{\nu})$ such that there exists a round $t\geq 0$ for which
$\pmb{\ell}_{M_{\textsl{WL}}}^{(t)}\not\sqsubseteq \pmb{\ell}_M^{(t)}$ holds. We
construct such an $M$ originating from a GCN~\cite{kipf-loose} defined in
Example~\ref{ex:KipfasMPNN}. That is, $M$ is a dMPNN in
$\mathcal{M}_{\textsl{dGNN}_4}$. Consider the labelled graph $( G,\pmb{\nu})$ with
vertex labelling $\pmb{\nu}_{v_1}=\pmb{\nu}_{v_2}=(1,0,0)$,
$\pmb{\nu}_{v_3}=\pmb{\nu}_{v_6}=(0,1,0)$ and
$\pmb{\nu}_{v_4}=\pmb{\nu}_{v_5}=(0,0,1)$, and edges $\{v_1,v_3\}$, $\{v_2,v_3\}$,
$\{v_3,v_4\}$, $\{v_4,v_5\}$, and $\{v_5,v_6\}$, as depicted in
Figure~\ref{fig:graph1}.
\begin{figure}[ht]
\centering
\begin{tikzpicture}
\node[fill=black,circle,inner sep=2pt,align=center,label=below:{$v_1$}] (v1) {};
\node[below=0.5cm of v1,fill=black,circle,inner sep=2pt,align=center,label=below:{$v_2$}] (v2) {};
\node[below right = 0.25cm and 1cm of v1,fill=black,circle,inner sep=2pt,align=center,label=below:{$v_3$}] (v3) {};
\node[right = 1cm of v3,fill=black,circle,inner sep=2pt,align=center,label=below:{$v_4$}] (v4) {};
\node[right = 1cm of v4,fill=black,circle,inner sep=2pt,align=center,label=below:{$v_5$}] (v5) {};
\node[right = 1cm of v5,fill=black,circle,inner sep=2pt,align=center,label=below:{$v_6$}] (v6) {};

\path
(v1) edge[-] (v3)
(v2) edge[-] (v3)
(v3) edge[-] (v4)
(v4) edge[-] (v5)
(v5) edge[-] (v6)
;
\end{tikzpicture}
\caption{Graph $G$.}\label{fig:graph1}
\end{figure}

Recall that $\pmb{\ell}^{(0)} = \pmb{\nu}$ and
\begin{equation*}
(\pmb{\ell}^{(1)}_M)_{v}:=\text{ReLU}\left(\left(\frac{1}{1+d_v}\right)\pmb{\ell}_v^{(0)}\mathbf{W}^{(1)} + \sum_{u\in N_G(v)} \left(\frac{1}{\sqrt{1+d_v}}\right)\left(\frac{1}{\sqrt{1+d_u}}\right)\pmb{\ell}^{(0)}_u\mathbf{W}^{(1)}\right).
\end{equation*}
We next define $\mathbf{W}^{(1)}:=\left(\begin{smallmatrix}
1 & 0 & 0\\
0 & 1 & 0\\
0 & 0 & 1
\end{smallmatrix}\right)$.
It can be verified that 
$$
\pmb{\ell}_M^{(1)}=
\begin{pmatrix}
\frac{1}{2} & 0\vphantom{\frac{1}{2\sqrt{2}}} & 0\\
\frac{1}{2} & 0\vphantom{\frac{1}{2\sqrt{2}}} & 0\\
0\vphantom{\frac{1}{2\sqrt{2}}} & \frac{1}{4} & 0\\
0\vphantom{\frac{1}{2\sqrt{2}}} & 0& \frac{1}{3} \\
0\vphantom{\frac{1}{2\sqrt{2}}} & 0 & \frac{1}{3}\\
0\vphantom{\frac{1}{2\sqrt{2}}} & \frac{1}{2}& 0
\end{pmatrix} +
\begin{pmatrix}
0 & \frac{1}{2\sqrt{2}} & 0\\
0 & \frac{1}{2\sqrt{2}} & 0\\
\frac{1}{\sqrt{2}} & 0 & \frac{1}{2\sqrt{3}}\\
0 & \frac{1}{2\sqrt{3}}& \frac{1}{3} \\
0 & \frac{1}{\sqrt{6}} & \frac{1}{3}\\
0 & 0& \frac{1}{\sqrt{6}}
\end{pmatrix} =
\begin{pmatrix}
\frac{1}{2} & \frac{1}{2\sqrt{2}}& 0\\
\frac{1}{2} & \frac{1}{2\sqrt{2}}& 0\\
\frac{1}{\sqrt{2}} & \frac{1}{4}& \frac{1}{
2\sqrt{3}}\\
0 & \frac{1}{2\sqrt{3}}& \frac{2}{
3}\\
0 & \frac{1}{\sqrt{6}}& \frac{2}{
3}\\
0 & \frac{1}{2}& \frac{1}{
\sqrt{6}}
\end{pmatrix}.
$$
We observe that $(\pmb{\ell}_M^{(1)})_{v_4}\neq (\pmb{\ell}_M^{(1)})_{v_5}$. We
note, however, that $
(\pmb{\ell}_{M_\textsl{WL}}^{(1)})_{v_4}=\textsc{Hash}\bigl((0,0,1),\ldbl(0,0,1),(0,1
,0)\rdbl\bigr)=(\pmb{\ell}_{M_\textsl{WL}}^{(1)})_{v_5}$. Hence,
$\pmb{\ell}_{M_{\textsl{WL}}}^{(1)}\not\sqsubseteq \pmb{\ell}_M^{(1)}$.
\end{proof}

The rest of this section is devoted to prove the second item in Theorem~\ref{thm:dmpnn}.
\begin{proposition}\label{prop:onestep}
$\architecture_{\textsl{deg}} \preceq_{+1} \architectureWL$
\end{proposition}

We will need the following lemma that states that anonymous MPNNs can compute the degrees of vertices in the first round of computation.

\begin{lemma}\label{lem:countdeg}
Let $( G,\pmb{\nu})$ be a labelled graph with $\pmb{\nu}:V\to \mathbb{A}^s$. There
exists an aMPNN $M_{d}$ such that
$(\pmb{\ell}_{M_d}^{(1)})_v=(\pmb{\nu}_v,d_v)\in\mathbb{A}^{s+1}$ for every vertex
$v$ in $V$.
\end{lemma}
\begin{proof}
We define the aMPNN $M_d$ with the following message and update functions. For each
$\mathbf{x}$, $\mathbf{y} \in \mathbb{A}^{s}$, $z\in\mathbb{A}$, and vertices $v,
u\in N_G(v)$ we define:
$$
\textsc{Msg}^{(1)}(\mathbf{x},\mathbf{y},-,-):= 1
\text{ and } \textsc{Upd}^{(1)}(\mathbf{x},z):=
\left(\mathbf{x},z\right).
$$
Then,
$
\mathbf{m}_v^{(1)}:=\sum_{u\in N_G(v)} 1 = d_v$ and $(\pmb{\ell}_{M_d}^{(1)})_v:=\textsc{Upd}^{(1)}(\pmb{\nu}_v,d_v)=(\pmb{\nu}_v,d_v)\in\mathbb{A}^{s+1}$,
as desired.
\end{proof}

We are now ready to prove Proposition~\ref{prop:onestep}.

\begin{proof}[Proof of Proposition~\ref{prop:onestep}]
By Theorem~\ref{thm:anonymous} it suffices to prove that the class
$\architecture_{\textsl{deg}}$ is weaker than $\architectureano$, with $1$ step
ahead. Let $( G,\pmb{\nu})$ be a labelled graph with
$\pmb{\nu}:V\to\mathbb{A}^{s_0}$. Take an arbitrary dMPNN $M_1$ such that for every
round $t\geq 1$ the message function is
$$
\textsc{Msg}_{M_1}^{(t)}(\mathbf{x},\mathbf{y},d_v,d_u) \in\mathbb{A}^{s_t'}$$
and $\textsc{Upd}_{M_1}^{(t)}(\mathbf{x},\mathbf{z})$ is the update function.

We construct an aMPNN $M_2$ such that $\pmb{\ell}_{M_2}^{(t+1)}\sqsubseteq
\pmb{\ell}_{M_1}^{(t)}$ holds, as follows. We denote the message and update
functions of $M_2$ by $\textsc{Msg}_{M_2}^{(t)}$ and $\textsc{Upd}_{M_2}^{(t)}$,
respectively. We will keep as an invariant {\bf(I1)} stating that for all $v$ if we
have $\mathbf{x}' = (\pmb{\ell}_{M_1}^{(t-1)})_v \in \mathbb{A}^{s_{t-1}}$ then
$\mathbf{x} = (\mathbf{x}',d_v) = (\pmb{\ell}_{M_2}^{(t)})_v \in \mathbb{A}^{s_{t-1}
+ 1}$.

For $t=1$, we let $\textsc{Msg}_{M_2}^{(1)}$ and $\textsc{Upd}_{M_2}^{(1)}$ be the
functions defined by Lemma~\ref{lem:countdeg}. As a consequence,
$(\pmb{\ell}_{M_2}^{(1)})_v=(\pmb{\nu}_v,d_v)\in\mathbb{A}^{s_0+1}$ for every vertex
$v$. We clearly have that $\pmb{\ell}_{M_2}^{(1)}\sqsubseteq \pmb{\ell}_{M_1}^{(0)}$
and the invariant {\bf(I1)} trivially holds.

For $t\geq 2$, we define the message and update functions of $M_2$ as follows:
\[
\textsc{Msg}_{M_2}^{(t)}(\mathbf{x},\mathbf{y},-,-):=\textsc{Msg}_{M_1}^{(t-1)}(\mathbf{x}',\mathbf{y}',x,y)
\]
where $\mathbf{x} = (\mathbf{x}',x)$ and $\mathbf{y} = (\mathbf{y}',y)$ and by
invariant {\bf (I1)} $x = d_v$ and $y = d_u$. Notice that the message function
remains anonymous as $d_u$ and $d_v$ are not obtained by setting $f(v)=d_v$ and
$f(u)=d_u$ but instead were computed once by the first message aggregation and
encoded in the labels of $v$ and $u$. The update function is defined as follows:
\[
\textsc{Upd}_{M_2}^{(t)}(\mathbf{x},\mathbf{z}):=\left(\textsc{Upd}_{M_1}^{(t-1)}(\mathbf{x}',\mathbf{z}'),x\right)\in\mathbb{A}^{s_{t-1}+1},
\]
where $\mathbf{x} = (\mathbf{x}',x)$ and by invariant {\bf (I1)} $x = d_v$. In other
words, in each round $t\geq 2$, $M_2$ extracts the degrees from the last entries in
the labels and simulates round $t-1$ of $M_1$. It is readily verified that
$\pmb{\ell}_{M_2}^{(t)}\sqsubseteq \pmb{\ell}_{M_1}^{(t-1)}$ for every $t$, as
desired and that the invariant {\bf (I1)} holds.
\end{proof}

In particular it follows from Proposition~\ref{prop:onestep} that for the dMPNN $M$
constructed in the proof of Proposition~\ref{prop:notweaker} it holds that
$\pmb{\ell}_{M_{\textsl{WL}}}^{(2)}\sqsubseteq \pmb{\ell}_M^{(1)}$.

\subsection{Graph neural network-based degree-aware MPNNs}
We next consider the relation between the WL algorithm and dMPNNs that originate
from graph neural networks as those listed in Table~\ref{tab:dMPNNs}. More
specifically, we consider the following general graph neural network architecture
\begin{equation}
\mathbf{L}^{(t)}:=\sigma\left(\mathbf{L}^{(t-1)}\mathbf{W}_1^{(t)}+\mathsf{diag}(\mathbf{g})(\mathbf{A}+p\mathbf{I})\mathsf{diag}(\mathbf{h})\mathbf{L}^{(t-1)}\mathbf{W}_2^{(t)} + \mathbf{B}^{(t)}\right), \label{eq:dGNN}
\end{equation}
where $p\in\mathbb{A}$ is parameter satisfying $0\leq p\leq 1$, $\mathbf{W}_1^{(t)}$
and $\mathbf{W}_2^{(t)}$ are learnable weight matrices in $\mathbb{A}^{s_{t-1}\times
s_{t}}$, $\mathbf{B}^{(t)}$ is a bias matrix consisting of $n$ copies of the same
row $\mathbf{b}^{(t)}$, and $\mathsf{diag}(\mathbf{g})$ and
$\mathsf{diag}(\mathbf{h})$ are positive diagonal matrices in $\mathbb{A}^{n\times
n}$ obtained by putting the vectors $\mathbf{g}$ and $\mathbf{h}$ in $\mathbb{A}^n$
on their diagonals, respectively. We only consider vectors $\mathbf{g}$ and
$\mathbf{h}$ which are \textit{degree-determined}. That is, when $d_v=d_w$ then
$\mathbf{g}_v=\mathbf{g}_w$ and $\mathbf{h}_v=\mathbf{h}_w$ for all vertices $v$ and
$w$. Furthermore, $\sigma$ is either the sign or ReLU non-linear activation function.

It is readily verified that all graph neural networks mentioned so far can be seen
as special cases of~(\ref{eq:dGNN}). Moreover, graph neural networks of the
form~(\ref{eq:dGNN}) can be cast as dMPNNs. We denote the resulting class of dMPNNs
by $\architecture_{\textsl{dGNN}}$. The reason that one obtains dMPNNs is because of
the degree-determinacy assumption. More specifically, degree-determinacy implies that
$$
\mathbf{g}=(g(d_{v_1}),g(d_{v_2}),\ldots,g(d_{v_n}))
\text{ and }
\mathbf{h}=(h(d_{v_1}),h(d_{v_2}),\ldots,h(d_{v_n}))
$$
for some functions $g:\mathbb{N}^+\to \mathbb{A}^+$ and
$h:\mathbb{N}^+\to\mathbb{A}^+$.

\begin{example}
The GCN architecture of~\cite{kipf-loose} corresponds to graph neural networks of
the form~(\ref{eq:dGNN}), with
$\mathbf{W}_1^{(t)}=\mathbf{0}\in\mathbb{A}^{s_{t-1}\times s_t}$, $p=1$,
$\mathbf{b}^{(t)}=\mathbf{0}\in\mathbb{A}^s$, and where $\mathbf{g}=\mathbf{h}$ are
defined by the function $g(n)=h(n)=(1+n)^{-1/2}$. \hfill$\blacksquare$
\end{example}

We define the class $\architecture_{\textsl{dGNN}}$ as the class of dMPNNs with
message and update functions of the form:
\begin{align}\textsc{Msg}^{(t)}\bigl(\mathbf{x},\mathbf{y},d_v,d_u)&:=\frac{1}{d_v}\left(
	\mathbf{x}\mathbf{W}_1^{(t)} + pg(d_v)h(d_v)\mathbf{x}\mathbf{W}_2^{(t)} \right)+ g(d_v)h(d_u)\mathbf{y}\mathbf{W}_2^{(t)} \label{eq:dGNNmsg}
\intertext{ and }
\textsc{Upd}^{(t)}(\mathbf{x},\mathbf{y}):=\sigma\left(\mathbf{y}\right) \label{eq:dGNNupd}
\end{align}
for any $\mathbf{x},\mathbf{y}\in\mathbb{A}^{s_{t-1}}$,
$\mathbf{W}_1^{(t)}\in\mathbb{A}^{s_{t-1}\times
s_t}$,$\mathbf{W}_2^{(t)}\in\mathbb{A}^{s_{t-1}\times s_t}$, bias vector
$\mathbf{b}^{(t)}\in\mathbb{A}^{s_t}$, and non-linear activation function $\sigma$.
We note that this encoding is just a generalisation of the encoding of GCNs as
dMPNNs given in Example~\ref{ex:KipfasMPNN}.

We know from Proposition~\ref{prop:onestep} that the class
$\architecture_{\textsl{dGNN}}$ is weaker than $\architectureWL$, with $1$ step
ahead. Indeed, it suffices to note that $\architecture_{\textsl{dGNN}}\subseteq
\architecture_{\textsl{deg}}$. In particular, the classes
$\architecture_{\textsl{dGNN}_1}$--$\architecture_{\textsl{dGNN}_6}$ corresponding
to the graph neural network architectures from Table~\ref{tab:dMPNNs} are all weaker
than $\architectureWL$, with $1$ step ahead. Furthermore, in the proof of
Proposition~\ref{prop:notweaker} we have shown that the condition that
$\architectureWL$ is $1$ step ahead is necessary for
$\architecture_{\textsl{dGNN}_4}$, and thus also for
$\architecture_{\textsl{dGNN}}$. We mentioned that one can provide similar examples
for $\architecture_{\textsl{dGNN}_2}$, $\architecture_{\textsl{dGNN}_5}$ and
$\architecture_{\textsl{dGNN}_6}$.

In contrast, we next show that the two remaining classes,
$\architecture_{\textsl{dGNN}_1}$ and $\architecture_{\textsl{dGNN}_3}$, are weaker
than $\architectureWL$ (with no step ahead). The reason is that dMPNNs in these
classes are equivalent to dMPNNs that only use degree information \textit{after}
aggregation takes places. These in turn are equivalent to anonymous MPNNs. We first
show a more general result, related to graph neural networks of the
form~\eqref{eq:dGNN} in which $\mathsf{diag}(\mathbf{h})=\mathbf{I}$. In other
words, the function $h:\mathbb{N}^+\to\mathbb{A}$ underlying $\mathbf{h}$ is the
constant one function, i.e., $h(n)=1$ for all $n\in\mathbb{N}^+$.

\begin{proposition}\label{prop:dGNNc}
The subclass of $\architecture_{\textsl{dGNN}}$, in which the function $h$ is the
constant one function, is weaker than $\architectureWL$.
\end{proposition}
\begin{proof}
We show that any MPNN $M$ in this class is an anonymous MPNNs. To see this, it
suffices to observe that any dMPNN in $\architecture_{\textsl{dGNN}}$, and thus also
$M$ in particular, is equivalent to a dMPNN with message and update functions
defined as follows. For every round $t \geq 1$, every
$\mathbf{x},\mathbf{y}\in\mathbb{A}^{s_{t-1}}$,
$\mathbf{z}=(\mathbf{z}',z)\in\mathbb{A}^{s_t+1}$, and every vertex $v$ and $u\in
N_G(v)$:
\begin{align}
\textsc{Msg}^{(t)}(\mathbf{x},\mathbf{y},d_v,d_u)&:=\left(h(d_u)\mathbf{y}\mathbf{W}_2^{(t)},1\right)\in\mathbb{A}^{s_t+1} \label{eq:h}
\intertext{ and }
\textsc{Upd}^{(t)}(\mathbf{x},\mathbf{z})&:=\sigma\left(\mathbf{x}\mathbf{W}_1^{(t)}+g(z)\mathbf{z}'+pg(z)h(z)\mathbf{x}\mathbf{W}_2^{(t)}+\mathbf{b}^{(t)}\right)\in\mathbb{A}^{s_t},
\end{align}
where $z\in\mathbb{A}$ will hold the degree information of the vertex under
consideration (i.e., $d_v$) after message passing. That is, we use a similar trick
as in Lemma~\ref{lem:countdeg}. Since we consider MPNNs in which $h(d_u)=1$, the
message function~(\ref{eq:h}) indeed only depends on $\mathbf{y}$. As a consequence,
$M$ is equivalent to an anonymous MPNN. From Theorem~\ref{thm:anonymous} and in
particular from $\architectureano\preceq \architectureWL$, the proposition follows.
\end{proof}

The architectures $\architecture_{\textsl{dGNN}_1}$ and
$\architecture_{\textsl{dGNN}_3}$ from Table~\ref{tab:dMPNNs} clearly satisfy the
assumption in the previous proposition and hence $\architecture_{\textsl{dGNN}_1},
\architecture_{\textsl{dGNN}_3} \preceq \architectureWL$.
  
We thus have shown the remaining part of the third item in Theorem~\ref{thm:dmpnn}.
\begin{corollary}\label{ex:landr}
The classes $\architecture_{\textsl{dGNN}_1}$ and $\architecture_{\textsl{dGNN}_3}$
are weaker than $\architectureWL$.
\end{corollary}

To conclude, we investigate whether $\architecture_{\textsl{dGNN}}$ and its
subclasses $\architecture_{\textsl{dGNN}_1}$--$\architecture_{\textsl{dGNN}_6}$ are
stronger than $\architectureWL$. For $\architecture_{\textsl{dGNN}}$ this follows
from Theorem~\ref{thm:anonymous}, stating in particular that
$\architecture_{\textsl{GNN}} \equiv \architecture_{\textsl{WL}}$, and from the
following remark.

\begin{remark}\label{rem:filips-second-try}
    It holds that $\architecture_{\textsl{GNN}} \preceq \architecture_{\textsl{dGNN}}$.
\end{remark}
Indeed, we first note that the class $\architecture_{\textsl{GNN}}$ is not a
subclass of $\architecture_{\textsl{dGNN}}$ since these classes differ in the
message and update functions used. We observe, however, that
$\architecture_{\textsl{GNN}}$ corresponds to the subclass of
$\architecture_{\textsl{dGNN}}$ in which the functions $g$ and $h$ are the constant
one function, i.e., $g(n)=h(n)=1$ for all $n\in\mathbb{N}^+$, and moreover, $p=0$.
More precisely, for every MPNN $M$ in $\architecture_{\textsl{GNN}}$ there is an
MPNN $M'$ in $\architecture_{\textsl{dGNN}}$ such that $M\equiv M'$, from which
Remark~\ref{rem:filips-second-try} follows.
 
So, we know already that $\architectureWL\preceq\architecture_{\textsl{dGNN}}$.
However, the aMPNN $M$ in $\architecture_{\textsl{GNN}}$ such that $M_{\textsl{WL}}
\preceq M$ holds, as constructed for Theorems~\ref{thm:grohe_lower}
and~\ref{thm:equalstrong}, does not comply with the forms of MPNNs corresponding to
the graph neural networks given in Table~\ref{tab:dMPNNs}. We next investigate which
classes $\architecture_{\textsl{dGNN}_i}$ are stronger that $\architectureWL$.

We start with some negative results, hereby showing part of the fourth item in Theorem~\ref{thm:dmpnn}.

\begin{proposition}\label{prop:notasstrong}
None of the classes $\architecture_{\textsl{dGNN}_i}$, for $i\in \{1,2, \dots, 5\}$, are stronger than $\architectureWL$.
\end{proposition}
\begin{proof}
The proof consists of a number of counterexamples related to the various classes of
dMPNNs under consideration. For convenience, we describe the counterexamples in
terms of graph neural networks rather than in their dMPNN form.

We first prove the proposition for classes of dMPNNs related to graph neural
networks of the form:
$$
\mathbf{L}^{(t)}:=\sigma\left(\mathsf{diag}(\mathbf{g})\mathbf{A}\mathsf{diag}(\mathbf{h})\mathbf{L}^{(t-1)}\mathbf{W}^{(t)}+\mathbf{B}^{(t)}\right).
$$
This includes $\architecture_{\textsl{dGNN}_i}$, for $i=1,2$. Consider the labelled
graph $( G_1,\pmb{\nu})$ with vertex labelling $\pmb{\nu}_{v_1}=(1,0,0)$,
$\pmb{\nu}_{v_2}=\pmb{\nu}_{v_3}=(0,1,0)$ and $\pmb{\nu}_{v_4}=(0,0,1)$, and edges
$\{v_1,v_2\}$, $\{v_1,v_3\}$, $\{v_4, v_2\}$ and $\{v_4,v_3\}$, as depicted in
Figure~\ref{fig:graphG1}.

\begin{figure}[ht]
\centering
\begin{tikzpicture}
\node[fill=black,circle,inner sep=2pt,align=center,label=below:{$v_2$}] (v2) {};
\node[below=1cm of v2,fill=black,circle,inner sep=2pt,align=center,label=below:{$v_3$}] (v3) {};
\node[below right = 0.5cm and 1cm of v2,fill=black,circle,inner sep=2pt,align=center,label=below:{$v_4$}] (v4) {};
\node[below left = 0.5cm and 1cm of v2,fill=black,circle,inner sep=2pt,align=center,label=below:{$v_1$}] (v1) {};

\path
(v1) edge[-] (v2)
(v1) edge[-] (v3)
(v4) edge[-] (v2)
(v4) edge[-] (v3)
;
\end{tikzpicture}
\caption{Graph $G_1$.}\label{fig:graphG1}
\end{figure}
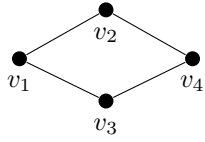

By definition, $\mathbf{L}^{(0)}:=\left(\begin{smallmatrix}1 & 0 &0\\
    0 & 1 &0\\
	0 & 1 &0\\
	0 & 0 &1 \end{smallmatrix}\right)$.
We note that $$(\pmb{\ell}_{M_{\textsl{WL}}}^{(1)})_{v_1}=\textsc{Hash}\bigl((1,0,0),\ldbl(0,1,0),(0,1,0)\rdbl\bigr)\neq
	(\pmb{\ell}_{M_{\textsl{WL}}}^{(1)})_{v_4}=\textsc{Hash}\bigl((0,0,1),\ldbl(0,1,0),(0,1,0)\rdbl\bigr).$$ We next show that there exist no
$\mathbf{W}^{(1)},\mathbf{B}^{(1)}$ such that $\mathbf{L}^{(1)}\sqsubseteq
\pmb{\ell}_{M_{\textsl{WL}}}^{(1)}$. Indeed, since the degree of all vertices is $2$
the computation is quite simple
	\allowdisplaybreaks
\begin{align*}
    \mathbf{L}^{(1)}& :=\sigma\left(\mathsf{diag}(\mathbf{g})\begin{pmatrix}0 & 1 & 1 &0 \\
    1 & 0 & 0 & 1\\
    1 & 0 & 0 & 1\\
    0 & 1 & 1 & 0\\
    \end{pmatrix}\mathsf{diag}(\mathbf{h})\mathbf{L}^{(0)}\mathbf{W}^{(1)} + \mathbf{B}^{(1)}\right)\\
 &    =\sigma\left(\begin{pmatrix}0 & g(2)h(2) & g(2)h(2) &0 \\
    g(2)h(2)& 0 & 0 & g(2)h(2)\\
    g(2)h(2) & 0 & 0 & g(2)h(2)\\
    0 & g(2)h(2) & g(2)h(2) & 0\\
    \end{pmatrix}
	\begin{pmatrix}1 & 0 &0\\
	    0 & 1 &0\\
		0 & 1 &0\\
		0 & 0 &1 \end{pmatrix}\mathbf{W}^{(1)} + \mathbf{B}^{(1)}\right)\\
	&=
    \sigma\left(\begin{pmatrix}
  0 &  2g(2)h(2)& 0\\
  g(2)h(2) & 0 & g(2)h(2)\\
  g(2)h(2) & 0 & g(2)h(2)\\
  0 &  2g(2)h(2)& 0\\  
    \end{pmatrix}\mathbf{W}^{(1)} + \mathbf{B}^{(1)}\right).
 \end{align*}
Finally, we recall that $\mathbf{B}^{(1)}$ consists of $n$ copies of the same
row. Hence, independently of the choice of $\mathbf{W}^{(1)}$ and
$\mathbf{B}^{(1)}$, vertices $v_1$ and $v_4$ will be assigned the same label,
and thus $\mathbf{L}^{(1)}\not\sqsubseteq \pmb{\ell}_{M_{\textsl{WL}}}^{(1)}$.
	
The second class of dMPNNs we consider are those related to graph neural networks of
the form:
$$
\mathbf{L}^{(t)}:=\sigma\left(\mathsf{diag}(\mathbf{g})(\mathbf{A}+\mathbf{I})\mathsf{diag}(\mathbf{h})\mathbf{L}^{(t-1)}\mathbf{W}^{(t)}+\mathbf{B}^{(t)}\right).
$$
This includes $\architecture_{\textsl{dGNN}_i}$, for $i=3,4$. Indeed, consider the
labelled graph $( G_2,\pmb{\nu})$ with one edge $\{v_1,v_2\}$, as depicted in
Figure~\ref{fig:graphG2}, and vertex labelling $\pmb{\nu}_{v_1}=(1,0)$ and
$\pmb{\nu}_{v_2}=(0,1)$.
    
\begin{figure}[ht]
\centering
\begin{tikzpicture}
\node[fill=black,circle,inner sep=2pt,align=center,label=below:{$v_1$}] (v1) {};
\node[right=1cm of v1,fill=black,circle,inner sep=2pt,align=center,label=below:{$v_2$}] (v2) {};

\path
(v1) edge[-] (v2)
;
\end{tikzpicture}
\caption{Graph $G_2$.}\label{fig:graphG2}
\end{figure}
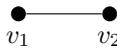
    
By definition, $\mathbf{L}^{(0)}:=\left(\begin{smallmatrix}1 & 0\\ 0 &
1\end{smallmatrix}\right)$. We also note that
$$(\pmb{\ell}_{M_{\textsl{WL}}}^{(1)})_{v_1}=\textsc{Hash}\bigl((1,0),\ldbl
(0,1)\rdbl\bigr)\neq
(\pmb{\ell}_{M_{\textsl{WL}}}^{(1)})_{v_2}=\textsc{Hash}\bigl((0,1),\ldbl(1,0)\rdbl
\bigr).$$ We next show that there exist no $\mathbf{W}^{(1)},\mathbf{B}^{(1)}$ such
that $\mathbf{L}^{(1)}\sqsubseteq \pmb{\ell}_{M_{\textsl{WL}}}^{(1)}$. Indeed,
	\allowdisplaybreaks
\begin{align*}
    \mathbf{F}^{(1)}& :=\sigma\Biggl(\mathsf{diag}(\mathbf{g})\biggl(\begin{pmatrix}0 & 1 \\
    1 & 0
    \end{pmatrix}+
    \begin{pmatrix}
    	 1 & 0 \\
    0 & 1 \end{pmatrix}\biggl)\mathsf{diag}(\mathbf{h})\mathbf{L}^{(0)}\mathbf{W}^{(1)} + \mathbf{B}^{(1)}\Biggr)\\
 &    =\sigma\Biggl(\begin{pmatrix}g(1) & 0 \\
    0 & g(1)
    \end{pmatrix}\begin{pmatrix}1 & 1 \\
    1 & 1
    \end{pmatrix}\begin{pmatrix}h(1) & 0 \\
    0 & h(1)
    \end{pmatrix}
   \begin{pmatrix}1 & 0\\
       0 & 1\end{pmatrix}\mathbf{W}^{(1)} + \mathbf{B}^{(1)}\Biggr)\\
	&=
    \sigma\Biggl(\begin{pmatrix}
    g(1)h(1)& g(1)h(1)\\
    g(1)h(1) & g(1)g(1)\\
    \end{pmatrix}\mathbf{W}^{(1)} + \mathbf{B}^{(1)}\Biggr).
 \end{align*}
Hence, independently of the choice of $\mathbf{W}^{(1)}$ and $\mathbf{B}^{(1)}$,
both vertices will be assigned the same label, and thus
$\mathbf{L}^{(1)}\not\sqsubseteq \pmb{\ell}_{M_{\textsl{WL}}}^{(1)}$.
	
Finally, we deal with the class $\architecture_{\textsl{dGNN}_5}$, i.e., dMPNNs
related to graph neural networks of the form
$$
\mathbf{L}^{(t)}:=\sigma\left((\mathbf{D}^{-1/2}\mathbf{A}\mathbf{D}^{-1/2}+\mathbf{I})\mathbf{L}^{(t-1)}\mathbf{W}^{(t)} + \mathbf{B}^{(t)}\right).$$
We consider the labelled graph $( G_3,\pmb{\nu})$ with vertex labelling
$\pmb{\nu}_{v_1}=\pmb{\nu}_{w_2}=\pmb{\nu}_{w_3}=(1,0,0)$,
$\pmb{\nu}_{w_1}=\pmb{\nu}_{v_2}=\pmb{\nu}_{v_3}=(0,1,0)$ and $\pmb{\nu}_{v_4} =
\pmb{\nu}_{v_5} = \pmb{\nu}_{w_4} = \pmb{\nu}_{w_5} = (0,0,1)$ and edges
$\{v_1,v_2\}$, $\{v_1,v_3\}$, $\{v_1,v_4\}$, $\{v_1,v_5\}$ and $\{w_1,w_2\}$,
$\{w_1,w_3\}$, $\{w_1,w_4\}$, $\{w_1,w_5\}$, as depicted in Figure~\ref{fig:graphG3}.

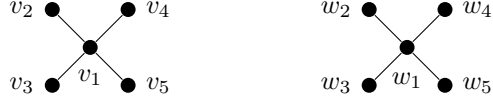
\begin{figure}[ht]
\centering
\begin{tikzpicture}[every node/.style={fill=black,circle,inner sep=2pt},node distance=0.5cm]

\node[label=below:{$v_1$}](v1) {};
\node[above left= of v1,label=left:{$v_2$}](v2) {};
\node[below left= of v1,label=left:{$v_3$}](v3) {};
\node[above right= of v1,label=right:{$v_4$}](v4) {};
\node[below right= of v1,label=right:{$v_5$}](v5) {};
\node[right=4cm of v1,label=below:{$w_1$}](w1) {};
\node[above left= of w1,label=left:{$w_2$}](w2) {};
\node[below left= of w1,label=left:{$w_3$}](w3) {};
\node[above right= of w1,label=right:{$w_4$}](w4) {};
\node[below right= of w1,label=right:{$w_5$}](w5) {};

\path
(v1) edge[-] (v2)
(v1) edge[-] (v3)
(v1) edge[-] (v4)
(v1) edge[-] (v5)
(w1) edge[-] (w2)
(w1) edge[-] (w3)
(w1) edge[-] (w4)
(w1) edge[-] (w5)
;
\end{tikzpicture}
\caption{Graph $G_3$.}\label{fig:graphG3}
\end{figure}

By definition, $\mathbf{L}^{(0)}:=\left(
\begin{smallmatrix}
1 & 0 & 0\\
0 & 1 & 0\\
0 & 1 & 0\\
0 & 0 & 1\\
0 & 0 & 1\\
0 & 1 & 0\\
1 & 0 & 0\\
1 & 0 & 0\\
0 & 0 & 1\\
0 & 0 & 1\end{smallmatrix}\right)$.
We also note that 
\begin{align*}
&(\pmb{\ell}_{M_{\textsl{WL}}}^{(1)})_{v_1}=\textsc{Hash}\bigl((1,0,0),\ldbl(0,1,0),(0,1,0),(0,0,1),(0,0,1)\rdbl\bigr)\\
\neq&(\pmb{\ell}_{M_{\textsl{WL}}}^{(1)})_{w_1}=\textsc{Hash}\bigl((0,1,0),\ldbl(1,0,0),(1,0,0),(0,0,1),(0,0,1)\rdbl\bigr).
\end{align*}
We next show that there exist no $\mathbf{W}^{(1)},\mathbf{B}^{(1)}$ such that $\mathbf{L}^{(1)}\sqsubseteq \pmb{\ell}_{M_{\textsl{WL}}}^{(1)}$.
Indeed,
	\allowdisplaybreaks
\begin{align*}
    \mathbf{L}^{(1)}& :=\sigma\left(\left(\mathsf{diag}\left(
    \begin{pmatrix}\frac{1}{2}\\
	1\\
	1\\
	1\\
	1\\
	\frac{1}{2}\\
	1\\
	1\\
	1\\
	1\\
	\end{pmatrix}\right)
	\begin{pmatrix}
	0 & 1 & 1 & 1 & 1 & 0 & 0 & 0 & 0 & 0\\
	1 & 0 & 0 & 0 & 0 & 0 & 0 & 0 & 0 & 0\\
	1 & 0 & 0 & 0 & 0 & 0 & 0 & 0 & 0 & 0\\
	1 & 0 & 0 & 0 & 0 & 0 & 0 & 0 & 0 & 0\\
	1 & 0 & 0 & 0 & 0 & 0 & 0 & 0 & 0 & 0\\
	0 & 0 & 0 & 0 & 0 & 0 & 1 & 1 & 1 & 1\\
	0 & 0 & 0 & 0 & 0 & 1 & 0 & 0 & 0 & 0\\
	0 & 0 & 0 & 0 & 0 & 1 & 0 & 0 & 0 & 0\\
	0 & 0 & 0 & 0 & 0 & 1 & 0 & 0 & 0 & 0\\
	0 & 0 & 0 & 0 & 0 & 1 & 0 & 0 & 0 & 0\\
	 \end{pmatrix}\mathsf{diag}\left(
	\begin{pmatrix}\frac{1}{2}\\
	1\\
	1\\
	1\\
	1\\
	\frac{1}{2}\\
	1\\
	1\\
	1\\
	1\\
	\end{pmatrix}\right)+\mathbf{I}\right)\mathbf{L}^{(0)}\mathbf{W}^{(1)} + \mathbf{B}^{(1)}\right)\\
	&=\sigma\left(\begin{pmatrix}
	0 & \frac{1}{2} & \frac{1}{2} & \frac{1}{2} & \frac{1}{2} & 0 & 0 & 0 & 0 & 0\\
	\frac{1}{2} & 0 & 0 & 0 & 0 & 0 & 0 & 0 & 0 & 0\\
	\frac{1}{2} & 0 & 0 & 0 & 0 & 0 & 0 & 0 & 0 & 0\\
	\frac{1}{2} & 0 & 0 & 0 & 0 & 0 & 0 & 0 & 0 & 0\\
	\frac{1}{2} & 0 & 0 & 0 & 0 & 0 & 0 & 0 & 0 & 0\\
	0 & 0 & 0 & 0 & 0 & 0 & \frac{1}{2} & \frac{1}{2} & \frac{1}{2} & \frac{1}{2}\\
	0 & 0 & 0 & 0 & 0 & \frac{1}{2} & 0 & 0 & 0 & 0\\
	0 & 0 & 0 & 0 & 0 & \frac{1}{2} & 0 & 0 & 0 & 0\\
	0 & 0 & 0 & 0 & 0 & \frac{1}{2} & 0 & 0 & 0 & 0\\
	0 & 0 & 0 & 0 & 0 & \frac{1}{2} & 0 & 0 & 0 & 0\\
	\end{pmatrix}
	\begin{pmatrix}
    1 & 0 & 0\\
    0 & 1 & 0\\
    0 & 1 & 0\\
    0 & 0 & 1\\
    0 & 0 & 1\\
    0 & 1 & 0\\
    1 & 0 & 0\\
    1 & 0 & 0\\
    0 & 0 & 1\\
    0 & 0 & 1
	\end{pmatrix}\mathbf{W}^{(1)}\right)
	 =\sigma\left(\begin{pmatrix}
	1 & 1 & 1\\
    \frac{1}{2} & 1 & 0\\
    \frac{1}{2} & 1 & 0\\
    \frac{1}{2} & 0 & 1\\
    \frac{1}{2} & 0 & 1\\
    1 & 1 & 1\\
    1 & \frac{1}{2} & 0\\
    1 & \frac{1}{2} & 0\\
    0 & \frac{1}{2} & 1\\
    0 & \frac{1}{2} & 1
	 \end{pmatrix}\mathbf{W}^{(1)} + \mathbf{B}^{(1)}\right).
 \end{align*}
Hence, independently of the choice of $\mathbf{W}^{(1)}$ and $\mathbf{B}^{(1)}$,
vertices $v_1$ and $w_1$ will be assigned the same label, and thus
$\mathbf{L}^{(1)}\not\sqsubseteq \pmb{\ell}_{M_{\textsl{WL}}}^{(1)}$.
\end{proof}

In particular, the class $\architecture_{\textsl{dGNN}_4}$, corresponding to the
popular graph neural networks of~\cite{kipf-loose}, is not stronger than
$\architectureWL$. We also remark, based on the first counterexample in the proof,
that the class of aMPNNs, corresponding to simple graph neural networks of the form
$\mathbf{L}^{(t)}:=\sigma(\mathbf{A}\mathbf{L}^{(t-1)}\mathbf{W}^{(t)}+\mathbf{B}^{(t
)})$, is not stronger than $\architectureWL$. We know, however, from
Corollary~\ref{cor:pluspstrongwl} that the slight extension
$\mathbf{L}^{(t)}=\sigma\left((\mathbf{A}+p\mathbf{I})\mathbf{L}^{(t-1)}\mathbf{W}^{(
t)}-q\mathbf{J}\right)$ results in a class of aMPNNs that is stronger than
$\architectureWL$. It will follow from our next result that a similar extension
suffices to make the graph neural networks of~\cite{kipf-loose} stronger than
$\architectureWL$.

We will now argue that the remaining $\architecture_{\textsl{dGNN}_6}$ class from
Table~\ref{tab:dMPNNs} is stronger than $\architectureWL$, hereby concluding the
proof of the fourth item in Theorem~\ref{thm:dmpnn}.

\begin{proposition}\label{prop:indeed-wl-power}
The class $\architecture_{\textsl{dGNN}_6}$ is stronger than $\architectureWL$.
\end{proposition}
\begin{proof}
We recall that dMPNNs in $\architecture_{\textsl{dGNN}_6}$ correspond to graph
neural network architectures of the form
\begin{align}\label{eq:GNN6}
\mathbf{L}^{(t)}&:=\sigma(\mathsf{diag}(\mathbf{g})(\mathbf{A}+p\mathbf{I})\mathsf{diag}(\mathbf{h})\mathbf{L}^{(t-1)}\mathbf{W}^{(t)}+\mathbf{B}^{(t)}),
\end{align}
where $\mathsf{diag}(\mathbf{g}) = \mathsf{diag}(\mathbf{h}) =
(r\mathbf{I}+(1-r)\mathbf{D})^{\nicefrac{-1}{2}}$ and $\sigma$ is ReLU or sign. In
fact, our proof will work for any degree-determined $\mathbf{g}$ and $\mathbf{h}$.

The argument closely follows the proof of Theorem~\ref{thm:equalstrong}. More
specifically, we construct a dMPNN $M$ corresponding to~(\ref{eq:GNN6}) such that
$\pmb{\ell}_M^{(t)}\sqsubseteq\pmb{\ell}_{M_{\textsl{WL}}}^{(t)}$ for all $t\geq 0$.
The induction hypothesis is that
$\pmb{\ell}_M^{(t)}\sqsubseteq\pmb{\ell}_{M_{\textsl{WL}}}^{(t)}$ and
$\pmb{\ell}_M^{(t)}$ is row-independent modulo equality. This hypothesis is clearly
satisfied, by definition, for $t=0$.

For the inductive step we assume that
$\pmb{\ell}_M^{(t-1)}\sqsubseteq\pmb{\ell}_{M_{\textsl{WL}}}^{(t-1)}$ and
$\pmb{\ell}_M^{(t-1)}$ is row-independent modulo equality. Let us define the
labelling $\pmb{\kappa}_M^{(t-1)}$ such that
$(\pmb{\kappa}_M^{(t-1)})_v:=h(d_v)(\pmb{\ell}_M^{(t-1)})_v$ for all vertices $v$.

\begin{lemma}\label{lem:indep}
We have that $\pmb{\kappa}_M^{(t-1)}\sqsubseteq\pmb{\ell}_{M_{\textsl{WL}}}^{(t-1)}$
and $\pmb{\kappa}_M^{(t-1)}$ is row-independent modulo equality.
\end{lemma}
\begin{proof}
Suppose that there are two vertices $v$ and $w$ such that
$$(\pmb{\kappa}_M^{(t-1)})_v=h(d_v)(\pmb{\ell}_M^{(t-1)})_v=h(d_w)(\pmb{\ell}_M^{(t-1)})_w=(\pmb{\kappa}_M^{(t-1)})_w.$$
This implies that $(\pmb{\ell}_M^{(t-1)})_v$ is a (non-zero) scalar multiple of
$(\pmb{\ell}_M^{(t-1)})_w$. This is only possible when
$(\pmb{\ell}_M^{(t-1)})_v=(\pmb{\ell}_M^{(t-1)})_w$ because $\pmb{\ell}_M^{(t-1)}$
is row-independent modulo equality. In other words,
$\pmb{\kappa}_M^{(t-1)}\sqsubseteq\pmb{\ell}_M^{(t-1)}\sqsubseteq\pmb{\ell}_{M_{\text
sl{WL}}}^{(t-1)}$. Similarly, suppose that $\pmb{\kappa}_M^{(t-1)}$ is not
row-independent modulo equality then, due to the definition of
$\pmb{\kappa}_M^{(t-1)}$, this implies that $\pmb{\ell}_M^{(t-1)}$ is also not
row-independent modulo equality.
\end{proof}

Lemma~\ref{lem:indep} gives us sufficient conditions to repeat a key part of the
argument in the proof of Theorem~\ref{thm:equalstrong}. That is, we can find a
matrix $\mathbf{U}^{(t)}$ such that the labelling
$\pmb{\mu}^{(t)}:v\mapsto\left((\mathbf{A}+p\mathbf{I})\mathsf{diag}(\mathbf{h})\mathbf{L}^{(t-1)}\mathbf{U}^{(t)}\right)_{v}$ satisfies $\pmb{\mu}^{(t)}\sqsubseteq
\pmb{\ell}_{M_{\textsl{WL}}}^{(t)}$. 

We will now prove that the labelling $\pmb{\lambda}^{(t)}$ defined by
$\pmb{\lambda}^{(t)}_v:=g(d_v)\pmb{\mu}^{(t)}_v$, also satisfies
$\pmb{\lambda}^{(t)}\sqsubseteq \pmb{\ell}_{M_{\textsl{WL}}}^{(t)}$. We remark that
$\pmb{\lambda}^{(t)}$ coincides with the labelling:
$$
v\mapsto \bigl(\mathsf{diag}(\mathbf{g})(\mathbf{A}+p\mathbf{I})\mathsf{diag}(\mathbf{h})\mathbf{L}^{(t-1)}\mathbf{U}^{(t)})_{v}.
$$

\begin{lemma}\label{lem:choose-p}
The exists a constant $m_p$, only dependent on $\mathbf{g}$ and the number $n$ of
vertices, such that $\pmb{\lambda}^{(t)}\sqsubseteq
\pmb{\ell}_{M_{\textsl{WL}}}^{(t)}$, for every $m_p<p<1$.
\end{lemma} 
\begin{proof}
We will choose $m_p$ at the end of the proof. For now suppose that
$\pmb{\lambda}^{(t)} \not \sqsubseteq \pmb{\ell}_{M_{\textsl{WL}}}^{(t)}$. Then
there exist two vertices $v$ and $w$ such that
$$
\pmb{\lambda}^{(t)}_v=\pmb{\lambda}^{(t)}_w \text{ and } (\pmb{\ell}_{M_{\textsl{WL}}}^{(t)})_v\neq (\pmb{\ell}_{M_{\textsl{WL}}}^{(t)})_w.$$ 
The latter implies that $\pmb{\mu}^{(t)}_v\neq \pmb{\mu}^{(t)}_w$ and thus
$\pmb{\lambda}^{(t)}_v=\pmb{\lambda}^{(t)}_w$ implies that $g(d_v)\neq g(d_w)$.

We recall some facts from the proof of Theorem~\ref{thm:equalstrong}, and from
equations~\eqref{eq:labelmu} and~\eqref{eq:linearcomb} in particular. An entry in
$\pmb{\mu}^{(t)}_v$ is either $0$ or $1,2,\ldots,n$ or $i+p$, for some $i\in
\{0,1,\dots,n\}$. Furthermore, at least one entry must be distinct from $0$. Also,
$\pmb{\lambda}^{(t)}_v=\pmb{\lambda}^{(t)}_w$ implies that the positions of the
non-zero entries in $\pmb{\mu}^{(t)}_v$ and $\pmb{\mu}^{(t)}_w$ coincide. (Recall
that the image of $g$ is $\mathbb{A}^+$). Let $ Z$ be the positions in
$\pmb{\mu}^{(t)}_v$ (and thus also in $\pmb{\mu}^{(t)}_w$) that carry non-zero
values.

We can now infer that $\pmb{\lambda}^{(t)}_v=\pmb{\lambda}^{(t)}_w$ implies that for
every $i\in Z$:
$$
\frac{\pmb{\mu}^{(t)}_{vi}}{\pmb{\mu}^{(t)}_{wi}}=\frac{g(d_w)}{g(d_v)}\neq 1.
$$
Moreover, both in $\pmb{\mu}^{(t)}_v$ and $\pmb{\mu}^{(t)}_w$ there are unique
positions $i_1$ and $i_2$, respectively, whose corresponding entry contain $p$. We
now consider three cases:
$$
\text{(a)~}
\frac{\pmb{\mu}^{(t)}_{vi_1}}{\pmb{\mu}^{(t)}_{wi_1}}=\frac{i+p}{j}\text{; \quad(b)~}
\frac{\pmb{\mu}^{(t)}_{vi_2}}{\pmb{\mu}^{(t)}_{wi_2}}=\frac{i}{j+p}\text{; \quad (c)~}
\frac{\pmb{\mu}^{(t)}_{vi_1}}{\pmb{\mu}^{(t)}_{wi_1}}=\frac{i+p}{j+p}\text{ (this is the case if and only if $i_1=i_2$),}
$$
for some $i,j\in \{0,1,2,\dots,n\}$. To define $m_p$, let
$\Gamma:=\left\{\frac{g(d_w)}{g(d_v)} \:\middle|\: g(d_v)\neq g(d_w) \text{ and }
v,w\in V\right\}$ and consider
	\begin{align*}
	P_a&:=\left\{\alpha j -i\hspace{0.5ex}\:\middle|\: 0\leq \alpha j -i < 1, i,j\in \{0,1,2\dots,n\},\alpha\in \Gamma \vphantom{\frac{1}{\alpha}}\right\}\\
	P_b&:=\left\{ \frac{i -\alpha j}{\alpha} \:\middle|\: 0\leq \frac{i -\alpha j}{\alpha} < 1, i,j\in \{0,1,2,\dots,n\},\alpha\in \Gamma\right\}\\
	P_c&:=\left\{ \frac{\alpha j-i}{1-\alpha} \:\middle|\: 0\leq \frac{\alpha(j-i)}{1-\alpha} < 1, i,j\in \{0,1,2,\dots,n\},\alpha\in \Gamma\right\}.
	\end{align*}
We define $m_p=\max\{P_1\cup P_2\cup P_3\cup\{0\}\}$ and we claim that for all $p$
satisfying $m_p<p<1$ the lemma holds.

By definition of $P_a$, $\alpha j-i \neq p$ and thus $\frac{i+p}{j} \neq \alpha$ for
any $\alpha\in\Gamma$ and $i,j\in\{0,1,\ldots,n\}$. This rules out (a). Similarly,
by definition of $P_b$, $\frac{i-\alpha j}{\alpha}\neq p$ and thus
$\frac{i}{j+p}\neq \alpha$ for any $\alpha\in\Gamma$ $i,j\in\{0,1,\ldots,n\}$. This
rules out (b). Finally, by definition of $P_3$, $\frac{\alpha j-i}{1-\alpha}\neq p$
and thus $\frac{i+p}{j+p}\neq \alpha$ for any $\alpha\in\Gamma$
$i,j\in\{0,1,\ldots,n\}$. This rules out (c). We conclude, as our initial assumption
cannot be valid for this $m_p$.
\end{proof}

From here, we can again follow the proof of Theorem~\ref{thm:equalstrong} to
construct a matrix $\mathbf{X}^{(t)}$ such that the labelling $\pmb{\ell}_M^{(t)}$
defined by
$\sigma(\mathsf{diag}(\mathbf{g})(\mathbf{A}+p\mathbf{I})\mathsf{diag}(\mathbf{h})
\mathbf{L}^{(t-1)}\mathbf{U}^{(t)}\mathbf{X}^{(t)}+\mathbf{B}^{(t)})$ with
$\mathbf{B}^{(t)}=-\mathbf{J}$ if $\sigma$ is the sign function, and
$\mathbf{B}^{(t)}=-q\mathbf{J}$ if $\sigma$ is the ReLU function, is such that
$\pmb{\ell}_M^{(t)}\sqsubseteq\pmb{\ell}_{M_{\textsl{WL}}}^{(t)}$ and
$\pmb{\ell}_M^{(t)}$ is row-independent modulo equality. This concludes the proof
for dMPNNs arising from graph neural networks of the form~\eqref{eq:GNN6}.
\end{proof}

We already mentioned that the proof of Proposition~\ref{prop:indeed-wl-power} works
for any degree-determined $\mathbf{g}$ and $\mathbf{h}$. In particular, the class of
dMPNNs originating from graph neural networks of the form
\begin{equation}
	\mathbf{L}^{(t)}:=\sigma\Bigl(\bigl(\mathbf{D}+\mathbf{I}\bigr)^{-1/2} (\mathbf{A}+p\mathbf{I})\bigl(\mathbf{D}+\mathbf{I}\bigr)^{-1/2} \mathbf{L}^{(t-1)}\mathbf{W}^{(t)} -q \mathbf{J}\Bigr),\label{gnn:kipfp}
\end{equation}
with $p,q\in\mathbb{A}$, $0\leq p,q\leq 1$, is stronger than $\architectureWL$. The
introduction of the parameter $p$ was already suggested in~\cite{kipf-loose}. The
proof of Proposition~\ref{prop:notweaker} shows that this parameter is necessary to
encode the WL algorithm. Our result thus provide a theoretical justification for
including this parameter.

\section{Conclusions}\label{sec:conclude}
In this paper we investigate the distinguishing power of two classes of MPNNs,
anonymous and degree-aware MPNNs, in order to better understand the presence of
degree information in commonly used graph neural network architectures. We show that
both classes of MPNNs are equivalent to the WL algorithm, in terms of their
distinguishing power, when one ignores the number of computation rounds. Taking the
computation rounds into consideration, however, reveals that degree information may
boost the distinguishing power.

Furthermore, we identify classes of MPNNs corresponding to specific
linear-algebra-based architectures of graph neural networks. We again distinguish
between anonymous graph neural networks~\cite{hyl17,grohewl} and degree-aware graph
neural networks~\cite{kipf-loose,DBLP:journals/corr/abs-1905-03046}. Here, we again
make connections to the WL algorithm, identify which architectures of graph neural
networks can or cannot simulate the WL algorithm, and describe how a simple
modification results in graph neural networks that are as powerful as the WL
algorithm.

Regarding future work, we point out that, following the work of~\cite{grohewl}, we
fix the input graph in our analysis. We use this particularly when we prove that
certain classes of MPNNs, based on graph neural network architectures, are stronger
than the WL algorithm (Theorem~\ref{thm:equalstrong} and
Proposition~\ref{prop:indeed-wl-power}). We prove this by constructing an MPNN that
simulates the WL algorithm on a fixed input graph. As such, the constructed MPNN may
not simulate the WL algorithm on another graph. It is natural to ask whether there
are graph neural network-based MPNNs that can simulate the WL algorithm on all
graphs.


\end{document}